\documentclass{article}

\usepackage{authblk}
\usepackage[a4paper, margin=1in]{geometry}
\usepackage[numbers, sort]{natbib}

\usepackage[utf8]{inputenc} 
\usepackage[T1]{fontenc}    
\usepackage{hyperref}       
\usepackage{url}            
\usepackage{booktabs}       
\usepackage{amsfonts}       
\usepackage{nicefrac}       
\usepackage{microtype}      
\usepackage{xcolor}         

\usepackage{algorithm}
\usepackage{algorithmic}
\usepackage{booktabs}

\usepackage{amsmath}
\usepackage{amssymb}
\usepackage{mathtools}
\usepackage{amsthm}
\usepackage{thm-restate}
\usepackage{mathrsfs}
\usepackage[inline]{enumitem}
\usepackage{multirow}
\usepackage{tikz}

\usepackage[most]{tcolorbox}

\usepackage{bbm}

\usepackage{nicefrac}
\usepackage{tikz}

\usepackage{array}
\newcolumntype{L}[1]{>{\raggedright\let\newline\\\arraybackslash\hspace{0pt}}m{#1}}
\newcolumntype{C}[1]{>{\centering\let\newline\\\arraybackslash\hspace{0pt}}m{#1}}
\newcolumntype{R}[1]{>{\raggedleft\let\newline\\\arraybackslash\hspace{0pt}}m{#1}}

\usetikzlibrary{shapes,decorations,arrows,calc,arrows.meta,fit,positioning}
\usepackage[acronym]{glossaries}
\hypersetup{hidelinks}
\usepackage{cleveref}

\makeglossaries
\newacronym{olu}{OLU}{online learner-unlearner}
\newacronym{ogd}{OGD}{Online Gradient Descent}
\newacronym{erm}{ERM}{Empirical Risk Minimisation}

\newcommand{\srdp}[3]{D_\alpha^{\br{#1}}\br{#2\Vert #3}}
\newcommand{\rdp}[2]{D_\alpha\br{#1\Vert #2}}
\newcommand{\regret}[2]{\operatorname*{Regret}_{#1}(#2)}

\newcommand{\Unaux}{\mathcal{U}_{\mathrm{aux}}}

\newcommand{\trainset}{\mathcal{S}}
\newcommand{\unlearnset}{\mathcal{S}^{\mathcal{U}}}
\newcommand{\trainsetsub}[1]{\trainset_{#1}}

\newcommand{\unlearnsetsub}[1]{\unlearnset_{#1}}
\newcommand{\OnUnalg}{\mathcal{A}_{\mathcal{R}}}
\newcommand{\UnDet}{h}
\newcommand{\UnNoise}{\rho}
\newcommand{\di}[1]{u[#1]}
\newcommand{\dt}[1]{\tau[#1]}

\usepackage{amartya_ltx}

\newtheorem{lemL}{Lemma}

\newtheorem{cor}{Corollary}
\Crefname{Corollary}{Corollary}{Corollary}
\Crefname{thm}{Theorem}{Theorem}
\Crefname{assumption}{Assumption}{Assumption}

\title{Online Learning and Unlearning}
\usepackage{times}
\author{Yaxi Hu$^1$}

\author{Bernhard Sch\"olkopf$^1$}
\author{Amartya Sanyal$^2$}
\date{}
\affil{$^1$MPI for Intelligent Systems, T\"ubingen, $^2$University of Copenhagen}

\begin{document}

\maketitle

\begin{abstract}
    We formalize the problem of online learning-unlearning, where a model is updated sequentially in an online setting while accommodating unlearning requests between updates. After a data point is unlearned, all subsequent outputs must be statistically indistinguishable from those of a model trained without that point. We present two online learner-unlearner (OLU) algorithms, both built upon online gradient descent (OGD). The first, \emph{passive OLU}, leverages OGD’s contractive property and injects noise when unlearning occurs, incurring no additional computation. The second, \emph{active OLU}, uses an offline unlearning algorithm that shifts the model toward a solution excluding the deleted data. Under standard convexity and smoothness assumptions, both methods achieve regret bounds comparable to those of standard OGD, demonstrating that one can maintain competitive regret bounds while providing unlearning guarantees.
\end{abstract}

\section{Introduction}\label{sec:intro}
Machine unlearning---the process of efficiently removing the influence
of specific training data so that a model behaves as if it were
retrained without that data---has recently attracted significant
attention. Beyond mitigating privacy risks (e.g., membership inference
attacks~\citep{shokri2017membershipinferenceattacksmachine,wang2023reconstructingtrainingdatamodel}),
unlearning algorithms are central to enforcing regulations such as the GDPR's
``right to be forgotten'' and can even improve model performance by
removing undesirable data points~\citep{goel2024correctivemachineunlearning}.

Most prior work on machine unlearning has focused on the
offline setting~\citep{neel2020descenttodeletegradientbasedmethodsmachine,
bourtoule2020machineunlearning, Sekhari21Remember,
mu2024rewindtodeletecertifiedmachineunlearning,
suriyakumar2022algorithms}, where data points are unlearned from an
already trained model, either all at once or sequentially. However, as
these methods operate \emph{after training is complete}, they do not
readily extend to settings where data arrives continuously. In
practice, data often arrives incrementally, and the model is updated
frequently with deletion requests interspersed among these updates.
This gives rise to new challenges that offline methods do not address. This observation motivates our study of an \emph{online learning and unlearning} framework. 

In this work, we introduce an online learning and unlearning framework that continually updates the model with incoming data while simultaneously accommodating deletion requests. We
formulate the online learning and unlearning problem and define an unlearning guarantee (inspired by the ``delete-to-control'' notion in~\citet{cohen2023controlconfidentialityrightforgotten}) tailored to
the online setting. Specifically, once a deletion request is
processed, all~\emph{future outputs of the algorithm} must not reveal
information about the deleted point. We formalise this in~\Cref{defn:online-learner-unlearner}.

To achieve this, we propose two algorithmic strategies:
\emph{(1)~Passive Unlearning} leverages intrinsic properties of online
algorithms which exhibit a Markovian Output
property~(Condition~\ref{condition:markov}) and
Contractiveness~(Condition~\ref{condition:contractive}),
to inject calibrated noise based on the time gap between learning and deletion. Notably~\Gls{ogd} exhibits both of these properties. This approach incurs no additional
computational cost relative to standard online learning.
\emph{(2)~Active Unlearning} exploits an auxiliary offline learning
algorithm (e.g., one based on Empirical Risk Minimisation (ERM)) with well-established unlearning algorithms, actively shifting the online
algorithm’s output towards that of the offline method. This strategy can
reduce the amount of noise required and thus potentially improve
regret bounds but comes with additional computational overhead.

For both strategies, we derive sub-linear regret bounds that closely
match the guarantees of standard~\Gls{ogd}. Roughly, for
convex cost functions with quadratic growth properties, passive unlearning achieves an expected regret of \(\bigO{\sqrt{T} + k^2 d}\) and for \(\mu\)-strongly convex losses, the regret bounds of both passive unlearning and active unlearning are \(\bigO{\log T + k^2d}\) for \(k\) deletion requests on a $d$-dimensional parameter space and \(T\) learning steps.

To summarise, our contributions are three‑fold:
\begin{itemize}
    \item \textbf{Problem formulation } We initiate the \emph{online learning and unlearning} setting and formalise an $(\alpha,\varepsilon)$‑online unlearning guarantee (\Cref{defn:online-learner-unlearner}) for this framework.
    \item \textbf{Algorithms } We design two algorithms—passive unlearner and active unlearner—both provably satisfy the above unlearning guarantee.
    \item \textbf{Theoretical guarantees } We prove that our algorithms add only constant‑factor computational overhead while attaining regret that nearly matches the best possible bound achievable without unlearning. A detailed comparison appears in \Cref{tab:comparison}.
\end{itemize}

The paper is organized as follows:~\Cref{sec:prelim} introduces the definitions of online convex optimization, machine unlearning, and online learning-unlearning;~\Cref{sec:passive,sec:active-unlearner} present the passive and active unlearners, with their theoretical guarantees;~\Cref{sec:discussion} concludes with performance insights.
\section{Preliminaries and Problem Setup}\label{sec:prelim}

In~\Cref{sec:preliminaries}, we review online convex optimization (OCO) and formally define unlearning. Then, in~\Cref{sec:OLU-defn}, we introduce the \emph{online learning and unlearning} framework. 

\vspace{-10pt}\subsection{Preliminaries}\label{sec:preliminaries}

\vspace{-5pt}\noindent\textbf{Online Convex Optimization } 
Let \(\cK\) denote a convex instance space and let \(\cF\) be a class
of convex cost functions mapping \(\cK\) to \(\bR_+\). 
OCO models an iterative game between a learner and an
 adversary over $T$ time steps. At each time step $t = 1, \ldots, T$,
 the learner $\cA$ selects a point $z_t \in \cK$, while the adversary chooses a convex cost function $f_t \in \cF$. The learner then incurs a cost $f_t(z_t)$.
 Formally, the learner’s update rule at each step is
 \begin{equation}
    \label{eq:defn-online-learning-updates}
    z_t = g_t(f_{1:t}, z_{1:t-1}). 
\end{equation}  
 where $g_t$ depends on all previously observed cost functions $f_{1},
 \ldots, f_{t}$ and past outputs $z_1, \ldots, z_{t-1}$. The
 performance of the learner is measured by its \emph{regret}:
\[
\regret{T}{\cA(f_{1:T})}= \sum_{t=1}^T f_t(z_t) - \min_{z\in \cK} \sum_{t=1}^T f_t(z),
\]  
where a sublinear regret is desirable. Throughout, we assume an
\emph{oblivious} adversary that fixes $f_1, \dots, f_T$ in advance,
and that $\cK$ has a bounded diameter $D$. Further assumptions on
$f_t$ (e.g.\ Lipschitzness, smoothness, or strong convexity;
see~\Cref{defn:Lipschitzness-smoothness}) appear in later sections as
needed. 

Online Gradient Descent (OGD) is a canonical algorithm for this setting. At time $t$, given a cost function \(f_t\) at time \(t\),~\Gls{ogd}
updates its output as:
\begin{equation}
    \label{eq:ogd-defn}
    z_t = \Pi_\cK\bs{z_{t-1} - \eta_t \nabla f_t\br{z_{t-1}}}.
\end{equation}
where \(\Pi_{\cK}\) is the projection operator onto the convex set \(\cK\). OGD achieves $O(\log T)$ regret for strongly convex and $O(\sqrt{T})$
regret for convex losses~\citep{elad2019introduction,orabona2019modern}.

\noindent\textbf{Unlearning } Machine unlearning aims to remove
\emph{post hoc} the influence of a specific training point on the
learned model. Naturally, the gold standard is to retrain from scratch
without that point and thus, a good unlearning procedure should
produce a model close to the retrained model. Various works use
different notions of statistical
indistinguishability~\citep{Guo19Certified,neel2020descenttodeletegradientbasedmethodsmachine,chien2024certified}
to formalize this closeness, inspired by differential
privacy~\citep{dwork2006differentialprivacy}. We
adopt the following definition of unlearning via R\'enyi divergence.

\begin{defn}
    \label{defn:unlearning-with-RDP}
    Let $\cA$ be a learning algorithm and $\cR$ an unlearning algorithm. For a dataset $\trainset$ and a subset $\unlearnset\subseteq \trainset$ of points to be removed, we say that $\cR$ is an $\br{\alpha, \varepsilon}$-unlearner if
  \[\rdp{\cR(\cA(\trainset), \unlearnset)}{\cR(\cA(\trainset\setminus \unlearnset), \emptyset)}\leq \varepsilon,\]
    where \(\rdp{\cdot}{\cdot}\)  is the $\alpha$-R\'enyi divergence.
  \end{defn}

  In most unlearning approaches, the unlearning step can be decoupled
  into a deterministic component which adjusts the current output to
  approximately match the output that would have been obtained had the
  algorithm never seen the unlearned point, and a noise component
  which adds calibrated noise to the adjusted output to obsfucate the
  approximation error. We define the deterministic unlearning function
  as \( \UnDet \) and the perturbation function as \( \UnNoise \).

\vspace{-10pt}
\subsection{Online Learning-Unlearning}\label{sec:OLU-defn}

\vspace{-5pt}We now integrate unlearning into the OCO framework. In an \emph{online learning-unlearning} game, the learner not only submits $z_t \in \cK$ each round and incurs $f_t(z_t)$, but may also receive requests to unlearn specific cost functions encountered in the past. Perhaps closest to our work, is the turnstile model in the continual observation literature, which similarly accommodates both insertion and deletion~\citep{jain2023turnstile}.

Let $k$ be the number of deletions, and let
\(
  \cU =\bc{\di{1},\dots,\di{k}}
  \quad\text{and}\quad
  \cT = \bc{\dt{1}, \ldots, \dt{k}}
\)
denote respectively the \emph{indices} of deleted functions and the
\emph{time steps} at which deletions occur (with $\di{i}\le \dt{i}$).
Thus, at time $\dt{i}$, the learner must ensure that the future
outputs are indistinguishable from those they would produce if
all functions $f_{\di{1}}, \ldots, f_{\di{i}}$ had never been observed. We call the learner in an online
learning-unlearning game an~\emph{\gls{olu}},
denoted by $\OnUnalg$.  

\noindent\textbf{Online Learner-Unlearner } An \gls{olu} can be constructed from a base online
learner. Formally, let $\cA$ be a base online learner with update functions
$g_1,\dots,g_T$. An \emph{online learner-unlearner} $\OnUnalg$
implements update function $g_t^{\OnUnalg}$ at time $t$,
\begin{equation}
    \label{eq:defn-online-learner-unlearner-update-func}
    g_t^{\OnUnalg} = \UnNoise_t\circ \UnDet_t \circ g_t.
\end{equation}
where $\UnDet_1,\dots,\UnDet_T$ are the deterministic unlearning
functions, and $\UnNoise_1,\dots,\UnNoise_T$ are the perturbation
functions. 

For two types of online learning algorithms,~\gls{olu} is constructed differently from~\Cref{eq:defn-online-learner-unlearner-update-func}.  
For base learners that explicitly use all past cost functions at
each step (e.g.\ FTL), unlearning ($\UnDet_t$ and $\UnNoise_t$) must be applied at
every round, essentially treating each update as an independent offline
unlearning problem. In contrast, incremental algorithms
such as OGD, whose update depends only on the previous model and current cost function, unlearning is invoked only at deletion rounds; otherwise, OLU mirrors OGD exactly. However, because past data are encoded in the memory of
the algorithm, unlearning can be more involved. In this paper, we
focus on the latter case.

\noindent\textbf{Certified OLU } Let \(\trainset = \{f_1, \dots, f_T\}\) be the cost functions chosen by the adversary, with \(\trainsetsub{t} = \{f_1, \dots, f_t\}\) as the subset up to time \(t\). The unlearned functions are \(\unlearnset = \{f_{\di{1}}, \dots, f_{\di{k}}\}\), where \(\unlearnsetsub{i} = \{f_{\di{1}}, \dots, f_{\di{i}}\}\) denotes the first \(i\) deletions. In the online
setting, removing a previously used function can shift indices for
future outputs. To manage this, we introduce a
skip element \( \perp \): an online learning algorithm does not update
its output at time steps where it encounters \( \perp \) and ignores
\( \perp \) in all future updates. The retraining dataset $\cS\setminus\cS^{\cU}$ is obtained by replacing every occurrence of a point $f \in\cS^{\cU}$ in $\cS$ with $\perp$.

\begin{defn}
    \label{defn:online-learner-unlearner}
    An OLU $\OnUnalg$ is an
    \emph{$(\alpha,\varepsilon)$-\gls{olu}} if,
    $\forall i=1,\dots,k-1$,
    \begin{equation}\label{eq:online-unlearning-defn}
        \rdp{\bs{\OnUnalg\br{\trainsetsub{\dt{i+1}-1}, \unlearnsetsub{i}, [\cT]_{1:i}}}_{\dt{i}:\dt{i+1}-1}}{\bs{\OnUnalg\br{\trainsetsub{\dt{i+1}-1}\setminus \unlearnsetsub{i}, \emptyset, [\cT]_{1:i}}}_{\dt{i}:\dt{i+1}-1}}\leq \varepsilon,
    \end{equation}
    where $\bs{\cdot}_{p:q}$ denotes the output sequence from time $p$ to $q$.
  \end{defn}

In~\Cref{defn:online-learner-unlearner}, we require that for each interval $[\dt{i},\dt{i+1})$ the OLU’s outputs be indistinguishable from those of retraining on the dataset with the first $i$ points removed. Although this condition is stated separately for each sub-interval, it nonetheless guarantees that once a point is deleted it remains protected forever. 
For example, when the second point $f_{\di{2}}$ is removed at $t = \dt{2}$, the interval-wise guarantees suffice to protect it for all future times: 
\begin{itemize}
    \item Interval $t\in  [\dt{2},\dt{3})$:~\Cref{defn:online-learner-unlearner} enforces indistinguishable from retraining on the dataset with the first two deleted points removed, thereby protecting $f_{\di{2}}$.
    \item Interval $t\in [\dt{3},\dt{4})$: after deleting $f_{\di{3}}$, \Cref{defn:online-learner-unlearner} enforces indistinguishability from retraining without the first three deleted points, which automatically continues to protect $f_{\di{2}}$.
    \item All later intervals: Each subsequent interval enforces indistinguishability from retraining with all points deleted up to that time, preserving protection of every previously removed point.
\end{itemize}



\noindent\textbf{OLU Regret }Since the best-in-hindsight comparator may change after
each deletion, we define regret in a manner reminiscent of dynamic regret ~\citep{Zinkevich2003,Zhao2020dynamic}, allowing the comparator to change across time: 
\begin{equation}\label{eq:regret-defn}
    \regret{T}{\OnUnalg\br{\trainset, \unlearnset, \cT}} = \sum_{i = 0}^k \sum_{t = \dt{i}+1}^{\dt{i+1}} 
    \Bigl[f_t(z_t) - f_t(z_i^\star)\Bigr],\quad z_i^\star 
    = \argmin_{z\in \cK} 
    \bc{\sum_{t=1}^{T} f_t(z) \;-\; \sum_{j=0}^{i} f_{\di{j}}(z)}.
\end{equation}
Here, we define $\dt{k+1} = T$ and $\dt{0} = 0$. The term \(z_i^\star\) is the best-in-hindsight estimator after the $i^{\it th}$ and before the $(i+1)^{\it th}$ deletion, computed over $\{f_1, ..., f_T\}$ with the first $i$ deleted cost functions $\{f_{\di{i}}, \ldots, f_{\di{i}}\}$ removed. This ensures both the learner and the comparator share the same history of cost functions, as in classical online learning setting. 


\noindent \textbf{OLU with DP} Differentially private (DP) online learning algorithms exhibits online unlearning guarantees, as shown in~\Cref{prop:DP-to-unlearner}. 
However, direct comparison is challenging because existing such methods~\citep{Smith13OptimalPrivateOnlineLearning,Prateek12DPOnline}    are typically analyzed under 
\((\varepsilon, \delta)\)-DP. Nonetheless, for the sake of comparison, interpreting these as 
\((\alpha, \varepsilon)\)-RDP online algorithms yields 
\((\nicefrac{\alpha}{k}, k^{1.6}\varepsilon)\)-\gls{olu}, as formalized in 
Proposition~\ref{prop:DP-to-unlearner}. This interpretation is fair as for any \(\delta>0\), any \((\alpha, \varepsilon)\)-RDP algorithm is also \(\br{2\varepsilon,\delta}\)-DP for \(\alpha>1 +\frac{\ln\nicefrac{1}{\delta}}{\varepsilon}\). However, this approach incurs a regret of \(O(dk \sqrt{T}/\varepsilon)\) for convex losses and 
\(O(dk\,\log^{2.5} T/\varepsilon)\) for strongly convex losses, where \(d\) is the instance-space 
dimension. By contrast, our unlearning algorithms ensure that the \(\log T\) term in the regret 
bounds does not depend on \(d\), as discussed in 
Sections~\ref{sec:passive}--\ref{sec:active-unlearner}.

\begin{restatable}{prop}{DPtoOnlineUnlearning}\label{prop:DP-to-unlearner}
    For \(k \in \mathbb{N}\), any \((\alpha, \varepsilon)\)-RDP online learning algorithm 
    is an \((\nicefrac{\alpha}{k}, k^{1.6}\varepsilon)\)-\gls{olu} for any deletion set \(U\) and 
    deletion-time set \(\mathcal{T}\) of size \(k\), if $\alpha \geq 2k$.
  \end{restatable}

\section{Passive Unlearning}\label{sec:passive}
In this section, we introduce a general passive online learner-unlearner. This approach incurs no extra
computational overhead while ensuring a regret bound that is
comparable to a standard online learning algorithm without unlearning.
We build upon a base online learner, and illustrate how injecting
properly calibrated noise at deletion time steps guarantee unlearning,
as formalised below.

\noindent\textbf{Passive \gls{olu} } Our passive \gls{olu}~(\Cref{alg:zeroth-order-general}) is constructed around a
base online learning algorithm \(\cA\) whose update functions are
\(g_1, \ldots, g_T\).  When no deletion request arrives, the algorithm
simply follows the updates of \(\cA\). Upon receiving a deletion
request at \(\dt{i} \in \cT\), the algorithm first performs the
standard update defined by \(\cA\), then adds calibrated noise \(\xi_i
\sim \cN(0, \sigma_i^2)\) to each coordinate of the output
\(z_{\dt{i}}\). The noise scale \(\sigma_i\) depends on the timing of
the deletion request \(\dt{i}\), the index of the deleted point
\(\di{i}\), which are considered as public information, and certain properties of both the cost functions \(f_t\)
and the update functions \(g_t\).  These properties are inputs to the
algorithm and ensure the unlearning guarantee.

Formally, the update function $g_t^{\OnUnalg}$
of~\Cref{alg:zeroth-order-general} at time $t$ is defined as
$g_t^{\OnUnalg}\coloneqq\UnNoise_t\circ \UnDet_t\circ g_t$ (as defined
in~\cref{eq:defn-online-learning-updates}), where $\UnDet_t(x) = x$
for all $t < T$ is the identity function, indicating no update is
performed apart from the base learner's update, and $\UnNoise_t(x)$
injects noise only if $t \in \cT$ is a deletion time:
\begin{equation}
    \label{eq:perturbator-passive}
    \UnNoise_t(x) = \begin{cases}
    x & t\notin \cT\\
    x + \cN(0, \sigma_i^2) & t=\dt{i}\in \cT
\end{cases}
\end{equation}

Algorithm~\ref{alg:zeroth-order-general} displays a pseudocode
of this procedure.  In~\Cref{sec:passive-unlearning}, we introduce
three conditions
(Condition \ref{condition:markov},\ref{condition:contractive},\ref{condition:sensitivity})
on the base algorithm’s update functions $g_t$ that suffice to
establish the online unlearning guarantee in
\Cref{defn:online-learner-unlearner}. Then, in~\Cref{sec:passive-regret}, we derive the regret bound when the base online learner is OGD~(\Cref{eq:ogd-defn}), a widely used algorithm satisfying these conditions.

\begin{algorithm}[htbp]
    \small
    \caption{\small General Passive~\gls{olu} for $\cA$}
    \label{alg:zeroth-order-general}
    \begin{algorithmic}[1]
        \REQUIRE Sensitivities $\Delta_{1:T}$, cost functions $f_{1:T}$, base updates $g_{1:t}$ of $\cA$, learning rates $\eta_{1:T}$, contractive coefficient $\gamma$, deletion time set $\cT$, deletion index set $\cU$, privacy parameters $\varepsilon, \alpha$, and a real number $\omega>1$.

        \STATE Initialise $z_1 \in \cK$. 
        \FOR{Time step $t = 2, \ldots, T$}
          \STATE $z_t \gets g_t(f_{1:t},\, z_{1:t-1})$
          \hfill \COMMENT{ For OGD,$\mathtt{z_t = \Pi_{\cK}\bs{z_{t-1} - \eta_t\nabla f_{t-1}(z_{t-1})}}$} 
          \IF{there exists \(i\) such that \(t = \dt{i} \in \cT\)}
          \STATE $\sigma_i \gets \sqrt{\frac{\omega i^{\omega}}{2(\omega - 1)\varepsilon}}\gamma^{\dt{i}-\di{i}}\Delta_{\di{i}}$
        \STATE $z_t \gets z_t + \xi_i,\quad \text{where} \xi_i \sim \cN\br{0, \sigma_i^2 \cI_d}$
            
          \ENDIF
          \STATE \textbf{Output} $z_t$ 
        \ENDFOR
    \end{algorithmic}
\end{algorithm}

\vspace{-5pt}\subsection{Unlearning guarantee}\label{sec:passive-unlearning}

We
first discuss the three sufficient conditions on the update function
$g_t$ of the base algorithm \(\cA\) under
which~\Cref{alg:zeroth-order-general} is an $(\alpha,
\alpha\varepsilon)$-\gls{olu}. For any $z_1,z_2,x \in \cK$, $z_{1:t-1}\in \cK^{t-1}$, any cost function $f\in \cF$ and $f_{1:t}\in \cF^t$, $\gamma\in (0, 1]$ and $\Delta_t < \infty$, the update function $g_t$ satisfies: 
\begin{alignat}{2}
    &\textbf{Markovian Output: } &&\quad g_t(f_{1:t}, z_{1:t-1}) = g_t(f_t, z_{t-1}), \label{condition:markov} \tag{C1}\\
    &\textbf{$\gamma$-Contraction:} &&\quad \norm{g_t(f, z_1) - g_t(f, z_2)}_2 \leq \gamma\,\norm{z_1 - z_2}_2, \label{condition:contractive} \tag{C2}\\
    &\textbf{$\Delta_{1:T}$-Sensitivity:} &&\quad \norm{g_t(f, x) - x}_2 \leq \Delta_t. \label{condition:sensitivity}\tag{C3}
\end{alignat}

Condition~\ref{condition:markov} ensures that update $g_t$ depends only on the latest cost function and previous output, allowing it to be expressed as $g_t(f_t, z_{t-1})$ and simplifying subsequent conditions. The contraction property in
Condition~\ref{condition:contractive} has been central in privacy
analyses of
DP-SGD~\citep{feldman18iteration,altschuler2023privacynoisystochasticgradient},
sampling~\citep{altschuler2022resolvingmixingtimelangevin}, and
generalisation bounds for
SGD~\citep{hardt2016trainfastergeneralizebetter}. Recently,
\citet{chien2024certified} utilised contraction in designing an
\emph{active} unlearning scheme for (noisy) SGD in the offline
setting; to the best of our knowledge, our work is the first that
applies this idea in an online framework of unlearning. Bounded
sensitivity in Condition~\ref{condition:sensitivity} requires that each single-step update does not lead to
arbitrarily large changes in the model. This is a standard assumption that holds in many settings, as any continuous function
defined on a bounded domain admits a finite sensitivity. Given these conditions, \Cref{thm:general-passive-unlearning} proves that~\Cref{alg:zeroth-order-general} is an
$(\alpha,\alpha\varepsilon)$-\gls{olu}. 

\begin{restatable}{thm}{PassiveUnlearning}\label{thm:general-passive-unlearning}
  Let $\Delta_{1:T}< \infty$ and $\gamma \in (0, 1]$. If for all $t
  \in [T]$, the update function $g_t$ of algorithm $\cA$ fulfills
  Conditions~(\ref{condition:markov}), (\ref{condition:contractive})
  and (\ref{condition:sensitivity}), then 
  Algorithm~\ref{alg:zeroth-order-general} instantiated with $\cA$ is an $(\alpha, \alpha\varepsilon)$-\gls{olu}.
\end{restatable}

\paragraph{Proof Sketch.} 
\noindent
Consider a single deletion request at time $\dt{1}$ for the point that
originally appears at time $\di{1}$. Let $\ell = \dt{1} - \di{1}$
denote the gap between the point’s first inclusion and its requested
removal. We compare two output sequences: $(z_t)$, the usual updates by the online learning algorithm $\cA$ on all cost functions, and $(z_t')$, the updates when
$f_{\di{1}}$ is omitted. As illustrated
in~\Cref{fig:visualization-passive}, for $t < \di{1}$, neither process
has used $f_{\di{1}}$, so $z_t = z_t'$. At $t = \di{1}$, $(z_t)$ takes
one gradient step on $f_{\di{1}}$ while $(z_t')$ does not, creating a
maximum difference of $\Delta_{\di{1}}$ (due to~\ref{condition:sensitivity}). For
subsequent steps, both follow the same $\gamma$-contractive updates,
shrinking their distance by a factor of $\gamma \le 1$ each time. By
$t = \dt{1}$, the distance is at most $\gamma^\ell \Delta_{\di{1}}$. Injecting
suitably calibrated Gaussian noise of scale proportional to
$\gamma^\ell \Delta_{\di{1}}$ makes the two processes statistically
indistinguishable under R\'enyi divergence (see
Lemma~\ref{lem:gaussian-distributions-renyi-divergence}), yielding the
$(\alpha, \alpha\varepsilon)$-\gls{olu} guarantee.

\noindent For subsequent deletions, we must track two sequences of
random variables where each evolve according to a deterministic
$\gamma$-contractive map with an added random noise term at specific
time steps.
Leveraging~\Cref{lem:contractiveness-privacy-amplification}, an
argument similar to the Privacy Amplification by Iteration
in~\citet{feldman18iteration}, we ensure the guarantee holds after
each deletion. See the full proof in~\Cref{app:passive-unlearning}.  \hfill\ensuremath{\square}

\noindent
Finally, OGD (Equation~\ref{eq:ogd-defn}) satisfies Conditions~\ref{condition:markov}, \ref{condition:contractive}, and \ref{condition:sensitivity} under standard smoothness and convexity assumptions. As a result, Algorithm~\ref{alg:zeroth-order-general} instantiated with OGD is an $(\alpha,\alpha\varepsilon)$-\gls{olu}.

\begin{figure}[t]
    \centering
    {\small
    \begin{tikzpicture}[scale = 0.9]
        \node[fill=blue!20, circle, inner sep = 5pt] (z1) at (0, 0) {$z_1$};
        \node[below=5pt of z1]{$z_1'$};
        \node at (0, -0.6) [rotate=90] {$=$};
    
        \node[fill=blue!20, circle, inner sep = 5pt] (z2) at (1.4, 0) {$z_2$};
        \node[below=5pt of z2]{$z_2'$};
        \node at (1.4, -0.6) [rotate=90] {$=$};
    
        \draw[->] (z1) -- (z2); 
    
        \node[fill=blue!20, circle, inner sep = 0pt] (zu1-1) at (2.9, 0) {$z_{\di{1}-1}$};
        \node[below=5pt of zu1-1, inner sep = 3pt]{$z_{\di{1}-1}'$};
        \node at (2.9, -0.6) [rotate=90] {$=$};
    
        \node at (2.1, 0) {$\ldots$}; 
    
        \node[fill=blue!20, circle, inner sep = 3.8pt] (zu1) at (4.5, 0) {$z_{\di{1}}$};
        \node[fill=red!20, circle, inner sep = 2.9pt] (zu1') at (4.5, -2.5) {$z_{\di{1}}'$};
        
        \draw[->] (zu1-1) -- (zu1);
        \draw[->] (zu1-1) -- (zu1');
        \draw[<->, dashed] (zu1) -- (zu1') node[midway, right] {$L$}; 
    
        \node[fill=blue!20, circle, inner sep = 0pt] (zu1+1) at (6.1, 0) {$z_{\di{1}+1}$};
        \node[fill=red!20, circle, inner sep = 0pt] (zu1+1') at (6.1, -2.2) {$z_{\di{1}+1}'$};
    
        \draw[->] (zu1) -- (zu1+1);
        \draw[->] (zu1') -- (zu1+1');
        \draw[<->, dashed] (zu1+1) -- (zu1+1') node[midway, right] {$\gamma L$};
    
        \node[fill=blue!20, circle, inner sep = 0pt] (zu1+2) at (7.7, 0) {$z_{\di{1}+2}$};
        \node[fill=red!20, circle, inner sep = 0pt] (zu1+2') at (7.7, -1.9) {$z_{\di{1}+2}'$};
    
        \draw[->] (zu1+1) -- (zu1+2);
        \draw[->] (zu1+1') -- (zu1+2');
        \draw[<->, dashed] (zu1+2) -- (zu1+2') node[midway, right] {$\gamma^2 L$};
    
        \node[fill=blue!20, circle, inner sep = 0pt] (zt1-1) at (10, 0) {$z_{\dt{1}-1}$};
        \node[fill=red!20, circle, inner sep =0pt] (zt1-1') at (10, -1.5) {$z_{\dt{1}-1}'$};
        
        \node at (8.9, 0) {$\ldots$};
        \node[rotate=13] at (8.9, -1.7) {$\ldots$};
        \draw[<->, dashed] (zt1-1) -- (zt1-1'); 
        \node at (10.6, -0.7) {$\gamma^{\ell-1}L$};
    
        \node[fill=blue!20, circle, inner sep = 3.8pt] (zt1) at (12, 0) {$z_{\dt{1}}$};
        \node[fill=red!20, circle, inner sep = 2.9pt] (zt1') at (12, -1) {$z_{\dt{1}}'$};
    
        \draw[->] (zt1-1)--(zt1);
        \draw[->] (zt1-1')--(zt1');
        \draw[<->, dashed] (zt1) -- (zt1'); 
        \node at (12.5, -0.5) {$\gamma^{\ell}L$};

        \node at (13, 0) {$\ldots$}; 
        \node at (13, -1) {$\ldots$}; 

        \draw (zu1.north) .. controls (6, 1.5) and (10.5, 1.5) .. (zt1.north)
        node[midway, above=10pt] {$\ell$ steps of gradient update};
    \end{tikzpicture}}
    \caption{\small Visualization of the output sequence of algorithm $\cA$ up to the first deletion $\dt{1}$\vspace{-1cm}}
    \label{fig:visualization-passive}
\end{figure}
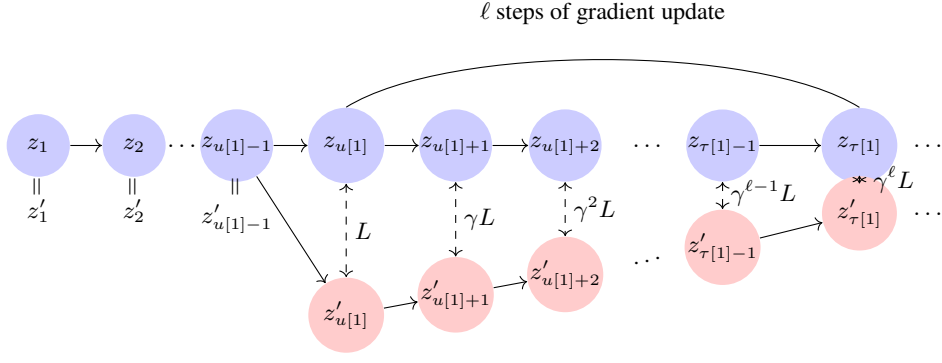



\begin{cor}\label{corollary:zeroth-order-unlearning} Assume each cost function $f_t$
    is $\beta$-smooth and convex. Then~\Cref{alg:zeroth-order-general}
    with OGD update step and learning rate $\eta \leq 2/\beta$ is an $(\alpha, \alpha \varepsilon)$-\gls{olu} for $\gamma = 1$. If the cost functions are
    $\beta$-smooth and $\mu$-strongly convex, then the same algorithm with $\eta\leq 1/(\beta + \mu)$
    is an $(\alpha, \alpha\varepsilon)$-\gls{olu} for $\gamma =
    \frac{\beta/\mu-1}{\beta/\mu + 1}$.
\end{cor}

\subsection{Regret guarantee}\label{sec:passive-regret}

In this section, we analyze the regret of our passive unlearning
algorithm with OGD as the base learner. First, we derive regret bounds
under decreasing learning rates for strongly convex~(\Cref{thm:regret-passive-unlearner-strongly-convex}) and convex~(\Cref{thm:regret-passive-unlearner-convex}) cost functions, highlighting the
role of the deletion index sets \(\cT,\cU\).
Next,~\Cref{thm:regret-passive-unlearner-convex-adaptive} uses
adaptive learning rates to obtain regret bounds based on the decay of
gradient norms. Finally, we show that a constant learning rate
achieves a worst-case regret guarantee (\Cref{thm:passive-worst-case})
for any deletion schedule, provided the total number of steps and
deletions are known.~\Cref{tab:learning-rate-comparison}
summarizes the regret guarantees
of~\Cref{alg:zeroth-order-general} under different learning rate
schedules and assumptions. 
\begin{table}[htbp]
    \centering
    {\small
    \renewcommand{\arraystretch}{1.2} 
    \begin{tabular}{l|c|c|c}
        \toprule
        \textbf{Method} & \textbf{Learning Rate} & \textbf{Assumptions} & \textbf{Regret Guarantees} \\
        \midrule
        \multirow{2}{*}{Decreasing} 
        & $\eta_t = \tfrac{1}{\mu t}$ 
        & Strongly convex (SC) 
        & \Cref{thm:regret-passive-unlearner-strongly-convex} \\
        &  $\eta_t = \tfrac{D}{L\sqrt{t}}$
        & Convex + Quadratic Growth (QG) 
        & \Cref{thm:regret-passive-unlearner-convex} \\
        \midrule
        Adaptive 
        & $\eta_t = \sqrt{\tfrac{D^2}{\sum_{i=1}^{t}\|\nabla f_i(z_i)\|_2^2}}$
        & Convex + Public Gradient Norms
        & \Cref{thm:regret-passive-unlearner-convex-adaptive} \\
        \midrule
        Constant 
        & $\eta = \sqrt{\frac{2D^2}{TL^2\br{1 + \frac{1.2\,k^{2.2}\,\,d}{0.42\varepsilon}}}}$ 
        & Knowledge of $k,T$; QG + Convex/SC
        & \Cref{thm:passive-worst-case} \\
        \bottomrule
    \end{tabular}}
    \caption{\small Overview of assumptions and regret bounds using~\Cref{alg:zeroth-order-general} with different learning rates.
    }
    \label{tab:learning-rate-comparison}
\end{table}


\noindent\textbf{Decreasing learning rate  }In passive unlearning, the final regret increases with as the deleted
point’s effect on subsequent outputs increases, since more noise must
be added to obscure its impact. Two competing factors related to the
time gap between the deletion request and the point’s initial
occurrence dictate this effect: \emph{(i)}~the decreasing learning
rate, which magnifies the impact of earlier points on future outputs,
and \emph{(ii)}~the contractiveness of gradient descent, which reduces
a point’s influence over
time.~\Cref{thm:regret-passive-unlearner-strongly-convex} presents the
regret of~\Cref{alg:zeroth-order-general} for strongly convex cost functions.

\begin{restatable}{thm}{PassiveRegretSC}
    \label{thm:regret-passive-unlearner-strongly-convex}
    Suppose $\cU$ (deletion indices) and $\cT$
    (deletion times) are each of size $k$. If the cost functions are
    $L$-Lipschitz, $\beta$-smooth, and $\mu$-strongly convex and $\di{i}\geq \frac{1}{2} + \frac{\beta}{\mu}$, then for
    all $T \ge k$, \Cref{alg:zeroth-order-general} with step size
    $\eta_t = 1/(\mu t)$ satisfies
      \[\bE\bs{\regret{T}{\trainset, \unlearnset, \cT}} = {\frac{L^2}{\mu}\br{\log T + 2k^2 + \frac{\sqrt{3}dk^{1.7}}{\varepsilon}\cG_1\br{\gamma, \cT,\cU}}}, \]
    where $\gamma = \frac{\beta/\mu - 1}{\beta/\mu + 1}$ and \(\cG_1\br{\gamma,\cT,\cU}=\sqrt{\sum_{i = 1}^k {\dt{i}^2\gamma^{4(\dt{i}-\di{i})}}/{\di{i}^4}}\). 
\end{restatable}

Here, the function $\cG_1$ reflects the effect of deletion indices on the regret: $\dt{i}^2/\di{i}^4$ captures the impact of the decreasing learning
rate (larger when $\di{i}$ is small), while
$\gamma^{4(\dt{i}-\di{i})}$ reflects contractiveness, which in the
strongly convex case ($\gamma < 1$) is the dominant term in the
expression. Therefore, each summand in $\cG_1$ is of the order $O(1)$ for arbitrary $\cT, \cU$ and we can get a loose regret bound of the order $O(\log T + k^2 + \nicefrac{dk^{2.2}}{\varepsilon})$ independent of the indices of deletion set. 


Next, \Cref{thm:regret-passive-unlearner-convex} addresses
the case of convex but not necessarily strongly convex cost functions.
Since the regret definition for the online learner-unlearner~(\Cref{eq:regret-defn}) involves a changing comparator, we
impose the Quadratic Growth (QG) assumption~(\Cref{assump:assumption-qg})---a weaker assumption than strong convexity~\citep{Chang2018Onthe}---on each aggregate cost function\(\sum_{t=1}^{\dt{i}} f_t\)\footnote{For the standard online learning problem, Theorem 3 in~\citet{Chang2018Onthe} implies that individual QG together with \Cref{assump:assumption2} yields an $O(\log T)$ regret bound, provided the learning rate decays rapidly ($\eta_t = O(1/t)$).  When the loss functions are only convex, the OLU setting requires a slower decay of the learning rate ($\eta_t = O(1/\sqrt{t})$) to keep the unlearning overhead small. Therefore, individual QG no longer suffices to achieve logarithmic regret in the OLU setting. 
}. The QG assumption on aggregate cost functions ensures that the comparator’s change after each unlearning is bounded.


\begin{assumption}[Quadratic Growth]\label{assump:assumption-qg}
    For any function $F$ on $\cK$, let $z^\star = \argmin_{z\in \cK} F(z)$.
    Then $F$ has quadratic growth with parameter $\kappa$ if for all
    $z \in \cK$,\[F(z) - F(z^\star) \geq \frac{\kappa}{2}\norm{z - z^\star}_2^2.\] 
\end{assumption}



\begin{restatable}{thm}{PassiveRegretCN}
    \label{thm:regret-passive-unlearner-convex}
   Suppose $\cU$ and $\cT$ are each of size $k$. If $f_1, \ldots, f_T$ are $L$-Lipschitz,
    $\beta$-smooth convex cost functions such that for each $i$, 
    $\sum_{t=1}^{\dt{i}}f_t$ satisfies \Cref{assump:assumption-qg} with
    parameter $\kappa(\dt{i}-\dt{i-1})$ and $\di{i} \geq \frac{\beta^2D^2}{4L^2}$, then for all $T \ge k$, \Cref{alg:zeroth-order-general}
    with $\eta_t = \tfrac{D}{L\sqrt{t}}$ satisfies
    \[\bE\bs{\regret{T}{\cR_{\cA}(\trainset,\unlearnset, \cT) }} \leq 3DL \sqrt{T}  + \frac{2k^2L^2 }{\kappa}+ \frac{3DLd k^{1.7}}{2\varepsilon}\cG_2(\cU, \cT),\]
    where $\cG_2(\cT, \cU) = \sqrt{\sum_{i = 1}^k\frac{\dt{i}}{\di{i}^2}}$. 
\end{restatable}

As
\Cref{thm:regret-passive-unlearner-convex} shows, if the function is only convex, we have
$\gamma=1$, so the effect of the decreasing learning rate dominates
over the contractiveness. Hence, deleting points that
occur later (large~\(\di{i}\)) lead to smaller regret. In particular,
if \( \dt{i} = o(\di{i}^2) \), $\cG_2(\cT, \cU)$ is constant and the regret remains \( O(\sqrt{T}) \),
matching the bound achieved by OGD without unlearning. 

The proof
of~\Cref{thm:regret-passive-unlearner-strongly-convex,thm:regret-passive-unlearner-convex}
follows the standard regret analysis of
OGD~\citep{hazan2007logarithmic}, with additional careful handling of
the noise term. See~\Cref{app:passive-regret} for the full proof.  


\noindent\textbf{Adaptive learning rates }
If we assume that the gradient norms at each output is public
information, one can use adaptive learning
rates~\citep{Duchi2011Adagrad,srebro2012optimisticrateslearningsmooth}.
Using this,~\Cref{thm:regret-passive-unlearner-convex-adaptive} achieves
a bound independent of \(\cT,\cU\), but dependent on how quickly these
norms decrease. 

\begin{restatable}{thm}{PassiveRegretCA}
    \label{thm:regret-passive-unlearner-convex-adaptive}
    Let $f_1,\dots,f_T$ be convex, $L$-Lipschitz and $\beta$-smooth
    and suppose \(\cU\) and \(\cT\) are each of size \(k\). Define $p(t)=\sum_{i=1}^t \|\nabla f_i(z_i)\|_2^2$, if there exists some $u_0 \geq 1$ such that $p(u_0)\geq\nicefrac{\beta^2}{4}$ and $\di{i}\geq u_0$, then the
    expected regret of~\Cref{alg:zeroth-order-general} with adaptive learning
    rate $\eta_t = \frac{D}{\sqrt{p(t)}}$ is
    \[\bigO{D^2\beta + D\sqrt{\sum_{i = 0}^k\sum_{t = \dt{i}+1}^{\dt{i+1}}f_t(z_i^\star)} + dk^2L^2D^2\cG_3(\cT, \cU, \trainset)},\]
    where $z_i^\star$ is a
    best-in-hindsight solution after the $k^\text{th}$ deletion, and $\cG_3(\cT, \cU, \trainset) = \sqrt{\beta \sum_{ i = 1}^k\frac{p(\dt{i})}{p(\di{i})^2} }$
\end{restatable}

The passive unlearner with an adaptive learning rate does not
explicitly require \( \dt{i} = \bigO{\di{i}} \) and instead accounts
for the algorithm's performance over time. Specifically, if the
post-deletion gradients do not grow significantly~\textit{i.e.}~\(\cG_3(\cT, \cU, \trainset) = O(\sqrt{T})\), the additional
regret from unlearning stays $O(\sqrt{T})$. Furthermore, the second term, \( D \sqrt{\sum_{i=0}^k \sum_{t=\dt{i}+1}^{\dt{i+1}} f_t(z_i^\star)} \), depends on the best-in-hindsight estimator and is \( O(\sqrt{T}) \) in the worst case. However, it can yield a tighter bound when the best-in-hindsight estimator incurs a smaller loss~\citep{srebro2012optimisticrateslearningsmooth}. Additionally, the adaptive method does not require prior knowledge of the Lipschitz constant of the cost functions.

\noindent\textbf{Constant learning rate }
The preceding results depend on either
\(\cT,\cU\)~(in~\Cref{thm:regret-passive-unlearner-convex,thm:regret-passive-unlearner-strongly-convex})
or the gradient
norms~(in~\Cref{thm:regret-passive-unlearner-convex-adaptive}).  Thus,
unfavourable deletion sets or cost functions can lead to high regret
well above the regret of their counterparts without unlearning.

In~\Cref{thm:passive-worst-case}, we show that using a \emph{constant}
learning rate ensures $O(k^{1.1}\sqrt{T})$ regret uniformly over any
deletion schedule for convex cost functions. While such a choice does
not rely on the timing of deletions or gradient shrinkage, it does
require knowledge of $T$ and $k$ beforehand and results in worse regret guarantee when the deletion schedule is favorable.  In practice, this may be
unreasonable or may necessitate a meta-strategy (e.g.\ doubling) to
tune the constant learning rate.

\begin{restatable}{thm}{PassiveRegretWorstCase}
    \label{thm:passive-worst-case}
    For any \(\cU,\cT\) of size $k$, if the cost
    functions $f_1, \ldots, f_T$ are convex, $L$-Lipschitz, and $\beta$-smooth, and
    $\sum_{t=1}^{\dt{i}}f_t$ satisfies \Cref{assump:assumption-qg} with
    parameter $\kappa(\dt{i}-\dt{i-1})$, then there exists a constant step size $\eta$ such that the expected regret of~\Cref{alg:zeroth-order-general} is
   \[L\br{D+ k^{1.1}\sqrt{\nicefrac{ d}{\varepsilon}}}\sqrt{2T} +  \frac{2L^2k^2}{\kappa}. \]
    \end{restatable}

\section{Active Unlearning}\label{sec:active-unlearner}
\vspace{-5pt}The passive approach in \Cref{sec:passive} passively exploits OGD's
properties without explicitly moving the current output toward the
retrained solution. Analyzing deterministic OGD updates directly for
unlearning can be difficult, and to the best of our knowledge, it has
not been done before.  Instead, we leverage the
\emph{descent-to-delete} method of
\citet{neel2020descenttodeletegradientbasedmethodsmachine}, originally
proposed for ERM, and integrate it into an active~\acrfull{olu} (\Cref{alg:ERM-first-order}). 

Though designed for ERM rather than OGD, the descent-to-delete
procedure can still bring our algorithm's output closer to the output
that would arise from retraining on data excluding the deleted points,
provided the current OGD output is not too far from the ERM minimizer.
As shown in \Cref{thm:first-order-guarantee}, this often allows the
active unlearning algorithm to add less noise than the passive method,
yielding improved regret bounds, as shown
in~\Cref{thm:first-order-guarantee}.

\noindent\textbf{Active~\Gls{olu} via Descent-to-Delete~\citep{neel2020descenttodeletegradientbasedmethodsmachine}} Of the various
ERM-based unlearning algorithms
\citep{Sekhari21Remember,neel2020descenttodeletegradientbasedmethodsmachine,suriyakumar2022algorithms},
we adapt the descent-to-delete approach of
\citet{neel2020descenttodeletegradientbasedmethodsmachine} thanks to
its simplicity, computational efficiency, and certified unlearning
guarantee.\footnote{Our framework also accommodates other ERM
unlearning schemes, e.g.\ the Newton-based method of
\citet{Sekhari21Remember}; see \Cref{alg:ERM-second-order}.} In
essence, this algorithm unlearns a set of points by running a few
gradient-descent iterations on the remaining data.~\Cref{alg:ERM-first-order} combines OGD as the base learner with
descent-to-delete as the deterministic unlearning function. The
procedure alternates between two modes: during regular learning,
it follows the OGD update rule.  When a request arrives to delete a
point $f_{\di{i}}$, the algorithm first performs $\cI_{1,i}$
gradient-descent steps on all previously seen cost functions, then
runs $\cI_{2}$ steps on all but the deleted points, and finally
injects calibrated noise.


Compared with the original offline descent-to-delete algorithm, our
online adaptation uses an additional $\cI_{1,i}$ steps of gradient
descent on \emph{all} previously seen cost functions. In offline
settings, the descent-to-delete procedure starts near the ERM solution
of the retained set; but OGD outputs can be far from this solution, so
the extra steps are required to bring the current model closer to the
ERM solution. This shift reduces the noise needed to ensure
unlearning and highlights an important caveat in adapting offline
unlearning algorithms to online settings.

Because the unlearning procedure moves the model away from the pure
OGD output toward an ERM solution, we require an additional
assumption~(\Cref{assump:assumption2}) to control this shift.
Intuitively, \Cref{assump:assumption2} ensures that each individual
cost function's minimizer is closely aligned with the overall ERM
minimizer, so gradient descent over these functions sequentially
pushes the model toward the global minimum.  \begin{assumption}
    \label{assump:assumption2}
    Let $\cT = \{\dt{1}, \dots, \dt{k}\}$, be the deletion times, $\dt{0} = 1$ and for each $1\le i\le k$, define 
    $z_i^\star = \argmin_z \sum_{t=1}^{\dt{i}} f_t(z).$
    Then, for every $i$, there exists $a_i\in \cK$ s.t. $\norm{a_i - z_i^\star} \le \frac{1}{\dt{i}}$
    and 
    \[
    \nabla f_t(a_i) = 0 \quad \text{for all } t \in (\dt{i-1},\dt{i}).\]
\end{assumption}

Under this assumption,
the active~\Gls{olu} in \Cref{alg:ERM-first-order} achieves $O(\log
T)$ regret \emph{independently} of the deletion indices $\cU$ but dependent on the deletion time $\cT$, as stated in~\Cref{thm:first-order-guarantee}.

\begin{restatable}{thm}{FirstOrderGuarantee}\label{thm:first-order-guarantee}
    Let $\cU$ and $\cT$ each have size $k$. For all $i$, assume
    $\dt{i-1}\le \di{i}\le \dt{i}$. Suppose each $f_t$ is
    $L$-Lipschitz, $\mu$-strongly convex, and $\beta$-smooth.  If the
    number of gradient-descent steps on all previously seen points,
    $\cI_{1,i}$ is at least \(\log_{\frac{1}{\gamma}} \frac{\mu D
    \dt{i}}{L}\) and $\cI_2\geq 2.2\log_{\frac{1}{\gamma}}k$, then \Cref{alg:ERM-first-order} is an
    $(\alpha,\alpha\varepsilon)$-\gls{olu}. Moreover,
    if \Cref{assump:assumption2} holds, the regret of
    \Cref{alg:ERM-first-order} is
    \[O\br{\log T+ k\br{LD^2 + \frac{Ld}{\mu \varepsilon}}+ \cG_2(\cT, \gamma) + \frac{L^2k^2}{\mu}},\]
    where $\gamma = \frac{\beta/\mu - 1}{\beta/\mu + 1}$ and $\cG_2 = \sum_{i = 1}^k \gamma^{\dt{i} - \dt{i-1}}\br{\dt{i} - \dt{i-1}}$.
\end{restatable}


\section{Discussion and Open Problems}\label{sec:discussion}




\begin{table}[t]
    \centering
    {\small\vspace{-10pt}
    \renewcommand{\arraystretch}{1.3} 
      \begin{tabular}{lcccc}
        \toprule
        \textbf{Algorithm} &\textbf{Assumptions} &\textbf{Regret} & $\cG$~(\textbf{Impact of deletion set}) &\textbf{Computation} \\
        \midrule
        \multirow{2}{*}{Passive~(\Cref{alg:zeroth-order-general})} & SC &  $\log T + k^2 + \cG$ & $dk^{1.7}\sqrt{\sum_{i = 1}^k\frac{\dt{i}^2\gamma^{4(\dt{i}-\di{i})}}{\di{i}^4}}$  & $1$\\
        &C + QG & $\sqrt{T} + k^2 + \cG$ & $dk^{1.7}\sqrt{\sum_{i = 1}^k\frac{\dt{i}}{\di{i}^2}}$& $1$ \\
        \midrule
        Active~(\Cref{alg:ERM-first-order}) &SC +~\ref{assump:assumption2} & $\log T + k^2 + \cG$ & $\sum_{i = 1}^k \gamma^{\Delta_\tau[i]}(\Delta_{\tau}[i])$ & $\log_{\frac{1}{\gamma}}\frac{k\mu D\dt{i}}{L}$ \\
        \midrule
        \multirow{2}{*}{Discard-and-restart}&SC  & $k\log T+ \cG$ &0& $1$ \\
                               &C   & $k\sqrt{T}+ \cG$&0 & $1$\\
        \midrule
        \multirow{2}{*}{Online DP} &SC & $dk\log^{2.5} T+ \cG$  &0 & $\log \dt{i}$\\
        & C  & $k\sqrt{dT\log^{2.5}T}+ \cG$  & 0&$\log \dt{i}$\\ 
        \midrule
        \multirow{2}{*}{Retraining }&SC  & $\log T + \cG$ & 0&  $\dt{i}$ \\
                       &C   & $\sqrt{T}+ \cG$ & 0&  $\dt{i}$\\
        \bottomrule
    \end{tabular}}
    \caption{\small Comparison of regret and computation cost of different~\gls{olu}: SC and C stands for strongly convex and convex setting, QG refers to~\Cref{assump:assumption-qg} over aggregate cost functions, and 2 refers to~\Cref{assump:assumption2}. Computation cost is with respect to each unlearning step. $\Delta_\tau[i] = \dt{i}-\dt{i-1}$. All values are expressed in order terms, omitting  constants and dependence on functional parameters like Lipschitzness for simplicity. \vspace{-10pt}}
    \label{tab:comparison}
\end{table}

\vspace{-5pt}\Cref{tab:comparison} compares our proposed unlearning algorithms to
several baselines in terms of regret and per-deletion computational
cost. Two naive baselines are \emph{retraining from scratch}, which
attains the best possible regret at $O(\dt{i})$ cost per deletion, and
\emph{discard-and-restart}, which reinitializes~\Gls{ogd} after every
deletion and trivially achieves $(\alpha,0)$-online unlearning
(\Cref{defn:online-learner-unlearner}). We also compare with the DP-online algorithm of \citet{Smith13OptimalPrivateOnlineLearning} as a
baseline despite it incurring at least $O(\log t)$ computation per step for both
learning and unlearning. We use the term \(\cG\) in~\Cref{tab:comparison} to specifically show how the nature of deletion requests~(i.e. \(\cT,\cU\)) affects the regret guarantees.

In the strongly convex setting with $\bigO{1}$ computational overhead per
unlearning, our passive~\Gls{olu} algorithm
(\Cref{alg:zeroth-order-general}) nearly matches the regret of
retraining, with the term $\cG$ decreasing exponentially in $\dt{i} - \di{i}$ i.e. the time between learning and deletion. However, its performance can worsen in the convex setting
if the adversary strategically selects earlier points for deletion i.e. small \(\di{i}\). In fact, unlike the strongly convex case, in the convex setting the additional term \(\cG\) does not vanish exponentially fast. By
allowing slightly more computation, $O(\log \dt{i})$ at each deletion
and assuming OGD converges to the ERM solution, the \emph{active}
approach (\Cref{alg:ERM-first-order}) can attain a regret bound that
essentially matches retraining for strongly convex losses, with the term $\cG$ decreasing exponentially with $\dt{i} - \dt{i-1}$ i.e. time between deletion requests. Both methods improve over the DP-based
algorithm by a factor of $d$ as well as a worse polylogarithmic dependence on \(T\)~(since unlearning imposes privacy only on
deleted data, unlike DP, which covers all points), and they also
outperform the discard-and-restart baseline by a factor of $k$. 

Although deriving a lower bound on regret in terms of $T$ would be valuable, we defer it to future work, since such a bound requires restricting the computation complexity of the unlearning algorithms. Without such constraint, FTL-type algorithms can achieve exact unlearning with the same regret as in the standard online learning setting.~\cite{ghazi2023ticketed} explores lower bounds on the space complexity of the exact unlearning in the offline setting, but it is orthogonal to establishing a regret lower bound. 


Compared to offline unlearning algorithms~\citep{chien2024certified,Guo19Certified,neel2020descenttodeletegradientbasedmethodsmachine,Sekhari21Remember}, our method requires weaker assumptions: while most offline approaches rely on strong convexity to establish unlearning guarantees, our passive \gls{olu} achieves this under merely convex losses~(\Cref{corollary:zeroth-order-unlearning}). However, we cannot fairly compare the accuracy of our algorithm to that of an offline method using online-to-batch conversion, because such conversions average over all past outputs, including those generated before the deleted point was removed, and therefore do not necessarily satisfy the unlearning guarantee.

To conclude, our unlearning methods offer preferable trade-offs between
computational efficiency and regret guarantees in the online setting.
Several open problems remain including setting lower bounds in this
problem, designing more efficient active~\Gls{olu} algorithms that do
not rely on strong convexity, and whether more unlearning friendly
online learning algorithms can be designed.

\bibliography{arxiv_main}

\appendix
\section{Omitted Proofs for~\Cref{sec:prelim}}

\begin{defn}\label{defn:Lipschitzness-smoothness}
    A function $f: \cX \to \cY$, is $L$-Lipschitz if the following hold for all $x, y\in \cX$, \[\norm{f(x) - f(y)}_2\leq L\norm{x - y}_2. \]
    $f$ is called $\mu$-strongly convex if for all $x, y\in \cX$, \[f(x) \geq f(y) + (\nabla f(y))^\top (x - y) + \frac{\mu}{2}\norm{x - y}_2^2.  \]
    $f$ is called $\beta$-smooth if for all $x, y \in \cX$, 
    \[f(x) \leq f(y) + (\nabla f(y))^\top (x - y) + \frac{\beta}{2}\norm{x - y}_2^2.  \]
\end{defn}

\begin{defn}\label{defn:renyi-divergence}
    For any two random variables $P, Q$ with corresponding distributions $\mu_P, \mu_Q$ respectively, and for any positive value $\alpha > 0, \alpha \neq 1$, the R\'enyi divergence between these two distributions is defined as 
    \[\rdp{\mu_P}{\mu_Q} = \frac{1}{\alpha - 1}\log \int \mu_P(x)^\alpha \mu_Q(x)^{1-\alpha} dx.\]
    For simplicity, we sometimes write $\rdp{P}{Q} = \rdp{\mu_P}{\mu_Q}$. 
\end{defn}

\DPtoOnlineUnlearning*

\begin{proof}
    Let the set of cost functions up to $\dt{i}$, with and without the deleted points indexed by $\cU$, be denoted by $\trainsetsub{i}$ and $\trainsetsub{i}'$ respectively, i.e. \[\trainsetsub{i} = \{f_1, ..., f_{\dt{i}}\}, \quad \trainsetsub{i}' = \{f_1, ..., f_{\dt{i}}\}\setminus \{f_{\di{j}}\}_{j = 1}^i. \]
    
    Then, for any $i\in \{1, \ldots, k\}$, the number of points that $\trainsetsub{i}$ and $\trainsetsub{i}'$ differ at is upper bounded by $k$. As $\cA$ is an online learning algorithm that is $(\alpha, \varepsilon)$-RDP, applying~\Cref{lem:group-privacy} on the dataset $\trainsetsub{i}$ and $\trainsetsub{i}'$, we have 
    \[D_{\frac{\alpha}{k}}(\cA(\trainsetsub{i})\Vert\cA(\trainsetsub{i}'))\leq k^{1.6}\varepsilon.\]
    \begin{lem}[Proposition 2 in~\citet{mironov2017renyi}]
        \label{lem:group-privacy}
        If an algorithm $\cA$ is $(\alpha, \varepsilon)$-RDP and if $\alpha \geq 2k$, then for any two dataset $S, S'$ differing by at most $k$ element, \[D_{\frac{\alpha}{k}}(\cA(S)\Vert \cA(S'))\leq k^{1.6}\varepsilon\]
    \end{lem}
    This concludes the proof. 
\end{proof}

\section{Omitted Proofs for~\Cref{sec:passive}}
\subsection{Unlearning guarantee of passive unlearning}\label{app:passive-unlearning}
We denote Euclidean norm by $\norm{\cdot}$ or $\norm{\cdot}_2$. 

\DPtoOnlineUnlearning*

\begin{proof}
    Let the set of cost functions up to $\dt{i}$, with and without the deleted points in $\cU$, be denoted by $\trainsetsub{i}$ and $\trainsetsub{i}'$ respectively, i.e. \[\trainsetsub{i} = \{f_1, ..., f_{\dt{i}}\}, \quad \trainsetsub{i}' = \{f_1, ..., f_{\dt{i}}\}\setminus \{f_{\di{j}}\}_{j = 1}^i. \]
    
    Then, for any $i\in \{1, \ldots, k\}$, the number of points that $\trainsetsub{i}$ and $\trainsetsub{i}'$ differ at is upper bounded by $k$. As $\cA$ is an online learning algorithm that is $(k\alpha, \varepsilon)$-RDP, applying~\Cref{lem:group-privacy} on the dataset $\trainsetsub{i}$ and $\trainsetsub{i}'$, we have 
    \[\rdp{\cA(\trainsetsub{i})}{\cA(\trainsetsub{i}')}\leq k^{1.6}\varepsilon.\]
    \begin{lem}[Proposition 2 in~\citet{mironov2017renyi}]
        \label{lem:group-privacy}
        If an algorithm $\cA$ is $(\alpha, \varepsilon)$-RDP and if $\alpha \geq 2k$, then for any two dataset $S, S'$ differing by at most $k$ element, \[D_{\frac{\alpha}{k}}(\cA(S)\Vert \cA(S'))\leq k^{1.6}\varepsilon\]
    \end{lem}
    This concludes the proof. 
\end{proof}

\begin{defn}[shifted R\'enyi divergence]
    Let $\mu$, $\nu$ be two distributions. For parameters $e\geq 0$ and $\alpha \geq 1$, the $e$-shifted R\'enyi divergence between $\mu$ and $\nu$ is defined as \[\srdp{e}{\mu}{\nu} = \inf_{\mu':W_{\infty}(\mu, \mu')\leq e}\rdp{\mu'}{\nu}, \] 
    where $W_{\infty}$ represents $\infty$-Wasserstein distance. 
    \end{defn}
    Shifted R\'enyi divergence satisfies monotonicity, i.e. For $0\leq e\leq e'$, $\srdp{e'}{\mu}{\nu}\leq \srdp{e}{\mu}{\nu}. $  
    For a distribution $\zeta$ and a vector $x$, we let $\zeta*x$ denote the distribution of $\eta + x$ where $\eta \sim \zeta$. We define \[R_\alpha(\zeta, \alpha) = \sup_{x:\norm{x}\leq a} \rdp{\zeta * x}{\zeta}.\]
    
\begin{defn}[Contractive Noise Iteration (CNI),~\citet{feldman18iteration}]\label{defn:CNI}
    Given an initial random state $X_0\in \cZ$, a sequence of contractive functions $\psi_t: \cZ\rightarrow \cZ$, and a sequence of noise distribution $\{\zeta_t\}$, we define the Contractive Noisy Iteration (CNI) by the following update rule: 
    \[X_{t+1} = \psi_{t+1}(X_t) + \xi_{t+1}, \]
    where $\xi_{t+1}$ is drawn independently from $\zeta_{t+1}$. We denote the random variable output by this process after $T$ steps as $CNI_T(X_0, \{\psi_t\}, \{\zeta_t\})$. 
    \end{defn}

\PassiveUnlearning*   
    \begin{proof}
        Let $\trainset = \{f_1, \ldots, f_T\}$ be the set of cost functions given to the learner over time, $\unlearnset$ be the set of deleted points with index in $\cU = \{\di{1}, \ldots, \di{k}\}$. Let $\cT = \{\dt{1}, \ldots, \dt{k}\}$ be the set of deletion times. Let $\trainset' = \{f_1', \ldots, f_T'\}$ with $f_t' = f_t$ at $t\notin\cT$ and $f_t' = \perp$ at $t\in \cT$. 
        
        We note that the output of~\Cref{alg:zeroth-order-general} at time $t$ is a CNI with a sequence of update functions $g_1, ..., g_t$, the noise distribution $\zeta_t$ is the Dirac delta distribution at 0 when there is no deletion request, $t\notin\cT$, and $\zeta_t = \cN(0, \sigma_i^2)$ for the $t = \dt{i}$. 
        
        The proof follows an application of~\Cref{lem:PAI}, a more general version of Theorem 22 in~\cite{feldman18iteration}. Compared with Theorem22 in their original paper, \Cref{lem:PAI} leverages the fact that the contractive coefficient \( \gamma \) is sometimes strictly less than 1, allowing us to achieve the same guarantee in terms of Rényi divergence by adding less noise. 
    
        \begin{lem}[Privacy amplification by iteration]\label{lem:PAI}
        Let $X_T$ and $X_T'$ denote the output of $CNI_T(X_0, \{\psi_t\}, \{\zeta_t\})$ and $CNI_T(X_0, \{\psi_t'\}, \{\zeta_t\})$, where $\psi_t, \psi_t'$ have contractive coefficient $\gamma$. Let $s_t = \sup_x \norm{\psi_t(x) - \psi_t'(x)}$. Let $a_1, ..., a_T$ be a sequence of reals and let $e_t = \sum_{i = 1}^t\gamma^{t - i}(s_i - a_i)$ such that $e_t \geq 0$ for all $t$, then \[\srdp{e_T}{X_T}{X_T'}\leq \sum_{t = 1}^TR_\alpha(\zeta_t, a_t). \]
        \end{lem}
    
        Let $\psi_t(z) = g_t(f_t, z)$ and $\psi_t'(z) = g_t(f_t', z)$ represents the update function with the $t$th cost functions from $\trainset$ and $\trainset'$ respectively.  Then, \[s_t = \sup_{z\in \cK}\norm{g(f_t, z) - g_t(f_t', z)} = \sup_{z\in \cK}\norm{\psi_t(z) - \psi_t'(z)} .\]  As all update functions $g_t$ are $\Delta_t$-bounded~(Condition \ref{condition:sensitivity}), we can compute the value of $s_t$ for all $t\in \{1, \ldots, T\}$, 
        
        \[s_t = \begin{cases} 
          \Delta_t & t \in \{\di{1}, ..., \di{k}\} \\
          0 & \text{otherwise} 
       \end{cases}\]

        Next, we select the sequence $a_1, ..., a_T$ such that $R_\alpha(\zeta_t, a_t)$ are bounded and $e_t = \sum_{i = 1}^t \gamma^{t-i}(s_i - a_i) \geq 0$ holds for all $t$. By definition of our algorithm~(\Cref{alg:zeroth-order-general}), the noise is only added at steps $t\in \cT$. Therefore, we need to set $a_t = 0$ for all $t\notin \cT$ to avoid unbounded $R_\alpha(\zeta_t, a_t)$ when $\zeta_t$ is a Dirac delta distribution at 0. Additionally, for $i\in \{1, \ldots, k\}$, we set $a_{\dt{i}} = \gamma^{\dt{i} - \di{i}} \Delta_{\di{i}}$, i.e. 
        \begin{equation}
            \label{eq:ais-passive}
            a_t = \begin{cases} 
          \gamma^{\dt{i} - \di{i}} \Delta_{\di{i}} & \text{if } t = \dt{i}, i \in \{1, \ldots, k\} \\
          0 & \text{otherwise} 
       \end{cases}.
        \end{equation}
        This ensures $e_{\dt{i}} = 0$ for all $\dt{i} \in \cT$ and $e_t \geq 0$ for all $t$. 

        For~\Cref{alg:zeroth-order-general} to satisfy the unlearning guarantee, it suffices to ensure the following indistinguishability condition holds at time \(\dt{i}\). Specifically, for each \(i \in \{1, \ldots, k\}\), we require that the R\'enyi divergence between the outputs of the two CNIs at step \(\dt{i}\) are bounded by \(\varepsilon\) according to~\Cref{defn:online-learner-unlearner}, i.e.,
        \[
        \rdp{X_{\dt{i}}}{X_{\dt{i}}'} \leq \alpha \varepsilon.
        \]
        To prove this, we apply~\Cref{lem:PAI} with the sequence $a_t$ selected above in~\Cref{eq:ais-passive}: for all $\dt{i}\in \cT$, where $i\leq k$, 
        \begin{equation}\label{eq:passive-unlearning-single-output}
        \begin{aligned}
            \rdp{X_{\dt{i}}}{X_{\dt{i}}'}&\leq \sum_{j = 1}^i R_\alpha(\zeta_{\dt{j}}, a_{\dt{j}})\\
            & \overset{(a)}{=} \sum_{j = 1}^i\frac{\alpha \br{a_{\dt{j}}^2}}{2\sigma_j^2}\overset{(b)}{=} \sum_{j = 1}^i \frac{\alpha \varepsilon}{j^\omega}\frac{ \omega-1}{\omega}\overset{(c)}{\leq} \alpha \varepsilon
            \end{aligned}
            \end{equation}
            where step (a) follows from~\Cref{lem:gaussian-distributions-renyi-divergence}, step (b) follows by the definition of $\sigma_j^2 = \frac{ j^\omega\omega(\gamma^{\dt{j} - \di{j}}\Delta_{\di{j}})^2}{2(\omega - 1)\varepsilon}$ in our algorithm, and step (c) follows \begin{equation*}
            \sum_{j = 1}^i \frac{ 1}{j^\omega} = 1 + \sum_{j = 2}^i\frac{1}{j^\omega}\leq 1 + \int_{1}^{\infty}\frac{1}{x^\omega}dx = 1 + \frac{1}{\omega -1} = \frac{\omega}{\omega -1}.
            \end{equation*}        
            
            \begin{lemL}[Corrolary 3 in \citet{mironov2017renyi}]\label{lem:gaussian-distributions-renyi-divergence}
                For any two Gaussian distributions of dimension $d$ with the same variance $\sigma^2I_d$ but different means $\mu_0, \mu_1$, denoted by $\cN\br{\mu_0, \sigma^2I_d}$ and $\cN(\mu_1, \sigma^2I_d)$, the following holds, 
                \[D_\alpha\br{\cN(\mu_0, \sigma^2I_d)||\cN(\mu_1, \sigma^2I_d)}\leq \frac{\alpha\norm{\mu_0 - \mu_1}^2}{2\sigma^2}.\]
            \end{lemL}
        Then the same guarantee in~\Cref{eq:passive-unlearning-single-output} extends to the output sequence between $t\in (\dt{i}, \dt{i+1})$ by post-processing property of R\'enyi divergence~(\Cref{lem:rdp-postprocessing}). Specifically, we consider a $(\dt{i+1}-1 - \dt{i})$-dimensional post-processing function $(I, \psi_{\dt{i}+1}, \psi_{\dt{i}+2}\circ \psi_{\dt{i}+1}, \ldots, \psi_{\dt{i+1}-1}\circ\ldots\psi_{\dt{i}+1})$. Applying~\Cref{lem:rdp-postprocessing}, we have $\rdp{X_{\dt{i}:\dt{i+1}-1}}{X_{\dt{i}:\dt{i+1}-1}'} \leq \alpha \varepsilon$ as desired. 
        \begin{lemL}[\citet{mironov2017renyi}]\label{lem:rdp-postprocessing}
            For any R\'enyi parameter $\alpha \geq 1$, any (possibly random) function $h$, and any two random variable $P, Q$ with corresponding distributions $\mu_P, \mu_Q$, 
            \[\rdp{h(P)}{h(Q)}\leq \rdp{P}{Q}.\]
        \end{lemL}
        
    \end{proof}

    \begin{proof}[Proof of~\Cref{lem:PAI}]
        The proof is by induction and similar to the original proof of Theorem 22 in~\cite{feldman18iteration}. 
    
        Let $X_t$ and $X_t'$ denote the t'th iteration of $CNI(X_0, \{\psi_t\}, \{\zeta_t\})$ and $CNI(X_0, \{\psi_t'\}, \{\zeta_t\})$ respectively. For each $t\leq T$, our goal is to show the following equation holds, \[\srdp{e_t}{X_t}{X_t'}\leq \sum_{i = 1}^tR_\alpha(\zeta_i, a_i). \]
        
        The base case follows by the definition that $e_0 = 0$ and $X_0 = X_0'$. For the induction step, let $\xi_{t+1}$ denote the random variable drawn from $\zeta_{t+1}$, 
        \begin{equation}
            \begin{aligned}
                \srdp{e_{t+1}}{X_{t+1}}{X_{t+1}'} &= \srdp{e_{t+1}}{\psi_{t+1}(X_t) + \xi_{t+1}}{\psi_{t+1}'(X_t') + \xi_{t+1}}\\
                &\overset{(a)}{\leq} \srdp{e_{t+1} + a_{t+1}}{\psi_{t+1}(X_t)}{\psi_{t+1}(X_t')} + R_\alpha(\zeta_{t+1}, a_{t+1})\\
                &\overset{(b)}{\leq} \srdp{\gamma e_t + s_{t+1}}{\psi_{t+1}(X_t)}{\psi_{t+1}(X_t')} + R_\alpha(\zeta_{t+1}, a_{t+1}) \\
                &\overset{(c)}{\leq} \srdp{e_t}{X_t}{X_t'} + R_\alpha(\zeta_{t+1}, a_{t+1}) \overset{(d)}{\leq} \sum_{i = 1}^{t+1} R_\alpha(\zeta_i, a_i),
            \end{aligned}
        \end{equation}
        where step (a) is due to~\Cref{lem:shift-reduction}, step (b) is due to definition of $e_t = \sum_{i = 1}^t\gamma^{t-i}\br{s_i - a_i}$, i.e. $e_t = \frac{e_{t+1} + a_{t+1}-s_{t+1}}{\gamma}$, step (c) follows from \Cref{lem:contractiveness-privacy-amplification}, a modification of privacy amplification of contraction (Lemma 21 in~\cite{feldman18iteration}), and step (d) follows by the induction hypothesis.

        \begin{lemL}[Shift-reduction lemma \citep{feldman18iteration}]\label{lem:shift-reduction}
        Let $\mu*\zeta$ denote the distribution of $X + Y$ where $X\sim \mu, Y\sim \zeta$.  Let $\mu, \nu$ and $\zeta$ be distributions. Then, for any $\alpha\geq 0$, 
        \[\srdp{e}{\mu*\zeta}{\nu*\zeta}\leq \srdp{e + a}{\mu}{\nu} + R_\alpha(\zeta, a), \]
        where $R_\alpha(\zeta, a) = \sup_{x:\norm{x}\leq a}D_\alpha(\zeta*x \Vert \zeta)$. 
        \end{lemL}
        \begin{lem}\label{lem:contractiveness-privacy-amplification}
                Suppose $\psi, \psi'$ are contractive maps with coefficient $\gamma$ and $\sup_x\norm{\psi(x) - \psi'(x)}\leq s$. Then, for r.v. $X$ and $X'$, \[\srdp{\gamma e + s}{\psi(X)}{\psi'(X)}\leq \srdp{e}{X}{X'}. \]
        \end{lem}
    
    \end{proof}
    \begin{proof}[Proof of~\Cref{lem:contractiveness-privacy-amplification}]
        The proof follows from that of Lemma 21 in~\cite{feldman18iteration}, by removing the step where the contractive coefficient $\gamma$ is upper bounded by $1$, and writing shift coefficient explicitly in terms of the contractive coefficient $\gamma$. We repeat their proof here for completeness. 
    
        By definition of $\srdp{e}{\cdot}{\cdot}$, there exists a joint distribution $(X, Y)$ such that \[\rdp{Y}{X'} = \srdp{e}{X}{X'}\text{ and }\bP\br{\norm{X - Y}\leq e} = 1.\] By the post processing properties of R\'enyi divergence (\Cref{lem:rdp-postprocessing}), \[\rdp{\psi'(Y)}{\psi'(X')}\leq \rdp{Y}{X'} = \srdp{e}{X}{X'}.\] Moreover, 
        \begin{equation}
            \begin{aligned}
                \norm{\psi(X) - \psi'(Y)}&\leq \norm{\psi(X) - \psi(Y)} + \norm{\psi(Y) - \psi'(Y)} \\
                &\overset{(a)}{\leq} \gamma \norm{X - Y} + s \leq \gamma e + s. 
            \end{aligned}
        \end{equation}
        where step (a) follows by the definition of contractive maps. 
    
        This shows that $(\psi(X), \psi'(Y))$ is a coupling establishing the claimed upper bound on \newline$\srdp{\gamma e + s}{\psi(X) }{\psi'(Y)}$, and concludes the proof. 
    \end{proof}
    
  
\begin{proof}[Proof of~\Cref{corollary:zeroth-order-unlearning}]
    Our result follows from the contractive properties of projected gradient descent. For standard gradient descent without projection, their contractive parameters are provided in~\Cref{lem:gd_contraction}. 
    
    \begin{lemL}[Proposition 18 \citep{feldman18iteration}; Lemma 2.2 \citep{altschuler2022resolvingmixingtimelangevin}]\label{lem:gd_contraction}
    Gradient descent step is contractive if the loss function is smooth or strongly convex. 
        \begin{itemize}
            \item Suppose a function $\ell: \bR^d\to \bR$ is convex, twice differentiable, and $M$-smooth. Then the function $\psi$ defined as $\psi(w) = w - \eta\nabla \ell(w)$ is contractive with parameter $(1-\eta M)$ for $\eta \leq 2/M$. 
            \[\norm{\psi(w) - \psi(w')} \leq \beta\norm{w - w'}\]
            \item Suppose $\ell$ is an $m$-strongly convex and $M$-smooth function for $0< m \leq M <\infty$. For step size $\eta = \frac{2}{M + m}$, then $\psi$ is contractive with parameter $\frac{M/m - 1}{M/m + 1}$.
        \end{itemize}
    \end{lemL}
    
    Since the projection operation is contractive, applying it after each gradient update preserves the contractive nature of gradient descent. Therefore, each step of projected online gradient descent remains contractive with the same coefficient as in the case of gradient descent without projection. This completes the proof. 
    \end{proof}

\subsection{Regret guarantee of passive unlearning}\label{app:passive-regret}
In this section, we present the proof of regret guarantees for the passive unlearning algorithm. 

\PassiveRegretSC*
\begin{proof}
    Recall that $\cT = \{\dt{1}, \ldots, \dt{k}\}$ represents the set of deletion times, and $\cU = \{\di{1}, \ldots, \di{k}\}$ represents the corresponding index of the points to be deleted. For a set of cost functions $\trainset = \{f_1, \ldots, f_T\}$, for any $i\in \{1, \ldots, k\}$, let $z_i^\star$ be the best-in-hindsight estimator after the $i$th deletion, i.e. \[z_i^\star = \arg\min_{z\in \cK} \sum_{t = 1}^{T} f_t(z) - \sum_{j = 1}^{i}f_{\di{j}}(z). \]    
    For the first part of the analysis, we consider the constant best-in-hindsight estimator \[z^\star = \arg\min_{z\in \cK} \sum_{t = 1}^{\dt{i}} f_t(z) . \]
    By the updating rule of~\Cref{alg:zeroth-order-general}, for $t+1 \notin \cT$,     
    \begin{equation}\label{eq:strongly-convex-zeroth-order1}
        \begin{aligned}
        \norm{z_{t+1} - z^\star}^2 &= \norm{\Pi_{\cK}\bs{z_{t} - \eta_{t}\nabla f_{t}(z_{t}) - z^\star}}^2\\
        &\leq \norm{z_{t} - z^\star }^2 + \eta_{t}^2\norm{\nabla f_{t}(z_{t})}^2 - 2\eta_{t} (\nabla f_{t}(z_{t}))^\top \br{z_{t} - z^\star}.
        \end{aligned}
    \end{equation}
    If there exists $i$ such that $t+1 = \dt{i}\in \cT$, 
    \begin{equation}\label{eq:strongly-convex-zeroth-order2}
        \begin{aligned}
        \norm{z_{t+1} - z^\star}^2 &= \norm{\Pi_{\cK}\bs{z_{t} - \eta_{t}\nabla f_{t}(z_{t}) + \xi_i - z^\star}}^2\\
        &\leq \norm{z_{t} - z^\star + \xi_i}^2 + \eta_{t}^2\norm{\nabla  f_{t}(z_{t})}^2 - 2\eta_{t} (\nabla f_{t}(z_{t}))^\top \br{z_{t} - z^\star + \xi_i}.
        \end{aligned}
    \end{equation}
    
    Rearrange~\Cref{eq:strongly-convex-zeroth-order1} and~\Cref{eq:strongly-convex-zeroth-order2}, we have 
    \begin{equation}\label{eq:strongly-convex-zeroth-order3}
    (\nabla f_{t}(z_{t}))^\top \br{z_{t} - z_i^\star} \leq \begin{cases} 
      \frac{\norm{z_{t}-z^\star}^2 - \norm{z_{t+1}-z^\star}^2}{2\eta_{t}} + \frac{\eta_{t}\norm{\nabla  f_{t}(z_{t})}^2}{2} & t+1 \notin \cT \\
      - \nabla f_{t}(z_{t})^\top \xi_i + \frac{\norm{z_{t}-z^\star + \xi_i}^2 - \norm{z_{t + 1}-z^\star}^2}{2\eta_{t}} + \frac{\eta_{t}\norm{\nabla  f_{t}(z_{t})}^2}{2}  &  t +1 = \dt{i} \in \cT 
    \end{cases}
    \end{equation}

    As the loss functions are $\mu$-strongly convexity and by the definition of strong convexity (\Cref{defn:Lipschitzness-smoothness}), \[ f_t(z_t) -  f_t(z^\star) \leq \br{\nabla  f_t (z_t)}^\top \br{z_t - z^\star} - \frac{\mu}{2}\norm{z_t - z^\star}^2.\]

    Summing over $t\in \{1, \ldots, T\}$ and substituting~\Cref{eq:strongly-convex-zeroth-order3} into the equation, 
    \begin{equation*}
        \begin{aligned}
            \regret{T}{\cR_{\cA}(\trainset, \emptyset, \cT)} &= \sum_{t = 1}^T   f_t(z_t) -  f_t(z^\star)\\
            &\leq \sum_{t = 1}^T \br{\nabla  f_t(z_t)}^\top (z_t - z^\star) - \frac{\mu}{2}\norm{z_t - z^\star}^2 \\
            &\overset{(a)}{=} \underbrace{\sum_{t = 1}^T \frac{\norm{z_{t}-z^\star}^2 - \norm{z_{t+1}-z^\star}^2}{2\eta_{t}} + \frac{\eta_{t}\norm{\nabla  f_{t}(z_{t})}^2}{2} - \frac{\mu}{2}\norm{z_t - z^\star}^2}_{A} \\
            & + \underbrace{\sum_{i = 1}^k - (\nabla f_{\dt{i}-1}(z_{\dt{i}-1}))^\top \xi_i + \frac{2\br{z_{\dt{i} - 1}-z^\star}^\top \xi_i + \norm{\xi_i}^2}{2\eta_{\dt{i}-1}}}_{B}
        \end{aligned}
    \end{equation*}
    where step (a) follows by substituting~\Cref{eq:strongly-convex-zeroth-order3}. 

    This way, we then decompose the expected regret into two parts. In the following, we derive separate upper bounds for each.
    \begin{equation}\label{eq:strongly-convex-zeroth-order4}
        \begin{aligned}
            \bE\bs{\regret{T}{\cR_{\cA}(S, \emptyset, \cT) }}&= \bE_{\xi_{1:k}}\bs{A + B}= \bE_{\xi_{1:k}}\bE\bs{A|\xi_{1:k}} + \bE_{\xi_{1:k}}\bs{B}
        \end{aligned}
    \end{equation}

    We first bound the term $\bE\bs{A|\xi_{1:k}}$. Given $\xi_{1:k}$, $z_1, ..., z_T$ are deterministic, we get
    \begin{equation*}
        \begin{aligned}
            \bE\bs{A|\xi_{1:k}} &= \sum_{t = 1}^T  \frac{\norm{z_{t}-z^\star}^2 - \norm{z_{t+1}-z^\star}^2}{2\eta_{t}} + \frac{\eta_{t}\norm{\nabla  f_{t}(z_{t})}^2}{2} -\frac{\gamma}{2}\norm{z_t - z^\star} \\
            &= \bE\bs{\sum_{t = 1}^T\br{\frac{1}{\eta_t}-\frac{1}{\eta_{t-1}}-\gamma}\norm{z_t - z^\star}^2 + \sum_{t = 1}^T \frac{\eta_t \norm{\nabla  f_t(z_t)}^2}{2}}\\
            &\overset{(a)}{=} \sum_{t = 1}^T\frac{\norm{\nabla f_t(z_t)}^2}{2\mu t} \overset{(b)}{\leq} \frac{L^2}{\mu}\br{1 + \log T}
        \end{aligned}
    \end{equation*}
    where step (a) is due to the definition of learning rate $\eta_t$, i.e. $\eta_t = \frac{1}{\mu t}$ and $\frac{1}{\eta_0} = 0$, and step (b) is due to the Lipschitzness of the cost function $ f_t$ and that $\sum_{t = 1}^T \frac{1}{t}\leq 1 + \log T$. 

    Thus, 
    \begin{equation}\label{eq:strongly-convex-zeroth-order5}
        \bE_{\xi_{1:k}}\bs{A} = \bE_{\xi_{1:k}} \bE\bs{A|\xi_{1:k}} \leq \frac{L^2}{\mu}(1 + \log T)
    \end{equation}

    It remains to bound $\bE_{\xi_{1:k}}\bs{B}$, 

    \begin{equation}\label{eq:strongly-convex-zeroth-order6}
        \begin{aligned}
            \bE_{\xi_{1:k}}\bs{B} &= \bE\bs{\sum_{i = 1}^k - (\nabla f_{\dt{i}-1}(z_{\dt{i}-1}))^\top \xi_i + \frac{2\br{z_{\dt{i} - 1}-z_i^\star}^\top \xi_i + \norm{\xi_i}^2}{2\eta_{\dt{i}-1}}}\\
            &\overset{(a)}{=}\bE\bs{\sum_{i = 1}^k \frac{ \norm{\xi_i}^2}{2\eta_{\dt{i}-1}}} = \sum_{i = 1}^k \frac{\bE\bs{\norm{\xi_i}}^2}{2\eta_{\dt{i}-1}} \\
            &\overset{(b)}{=} \sum_{i = 1}^k \frac{d\sigma_i^2}{2\eta_{\dt{i} -1}} \overset{(c)}{=} \frac{d\omega L^2}{4\varepsilon \mu \br{\omega -1}}\sum_{i = 1}^k\frac{\br{\dt{i} - 1}}{ \di{i}^2} \gamma^{2(\dt{i} - \di{i})}i^\omega
        \end{aligned}   
    \end{equation}
    where step (a) follows from the fact that $\xi_i$ are independent and have zero mean. Step (b) holds because $\bE(\norm{\xi_i}^2) = d\sigma_i^2$ for $\xi_i \sim \cN(0, \sigma_i^2 I_d)$. Step (c) follows by substituting $\sigma_i^2 = \frac{i^\omega \omega \alpha \gamma^{2(\dt{i} - \di{i})}\eta_{\di{i}}^2L^2}{2\varepsilon (\omega -1)}$, $\eta_{\di{i}} = \frac{1}{\di{i} \mu}$, along with $\eta_{\dt{i}} = \frac{1}{\dt{i} \mu}$. 

    To simplify the term $\sum_{i = 1}^k\frac{\br{\dt{i} - 1}}{ \di{i}^2} \gamma^{2(\dt{i} - \di{i})}i^\omega$, we apply Cauchy-Schwarz inequality, 
    \begin{equation}\label{eq:bounds-on-omega-passive-sc}
        \begin{aligned}
            \sum_{i = 1}^k\frac{\br{\dt{i} - 1}}{ \di{i}^2} \gamma^{2(\dt{i} - \di{i})}i^\omega&\leq \sqrt{\sum_{i = 1}^k \br{\frac{\br{\dt{i} - 1}}{ \di{i}^2} \gamma^{2(\dt{i} - \di{i})}}^2} \sqrt{\sum_{i = 1}^k i^{2w}}\\
            &\leq \sqrt{\sum_{i = 1}^k \br{\frac{\br{\dt{i} - 1}}{ \di{i}^2} \gamma^{2(\dt{i} - \di{i})}}^2}\frac{k^{\omega + 0.5}}{\sqrt{2\omega + 1}}\\
            &\leq \frac{k^{\omega + 0.5}}{\sqrt{3}}\sqrt{\sum_{i = 1}^k \br{\frac{\br{\dt{i} - 1}}{ \di{i}^2} \gamma^{2(\dt{i} - \di{i})}}^2}
        \end{aligned}
    \end{equation}
    where the last inequality follows from $\omega >1$. 
    
    Combining~\Cref{eq:strongly-convex-zeroth-order4,eq:strongly-convex-zeroth-order5,eq:strongly-convex-zeroth-order6,eq:bounds-on-omega-passive-sc}, we have 
    \begin{equation}\label{eq:strongly-convex-zeroth-order-final1}
        \bE\bs{\regret{T}{\cR_{\cA}(\trainset,\unlearnset, \cT) }} \leq \frac{L^2}{\mu}\br{1 + \log T} + \frac{d\omega k^{\omega + 0.5} L^2}{2\sqrt{3}\varepsilon \mu \br{\omega -1}}\sqrt{\sum_{i = 1}^k\br{\frac{\dt{i}}{ \di{i}^2} \gamma^{2(\dt{i} - \di{i})}}^2 }
    \end{equation}
    We note that up to this point in the proof, we have been computing the regret of our algorithm with respect to a constant best-in-hindsight estimator $z^\star$. However, this differs from our regret definition in~\Cref{eq:regret-defn}. 
    
    Recall that $\unlearnset$ is the set of points deleted by the algorithm with index in $\cU$. In the following, we bound the difference between these two regret measures: $\bE[\regret{T}{\cR_{\cA}(\trainset, \emptyset, \cT)}]$, which corresponds to the regret with a constant comparator, and $\bE[\regret{T}{\cR_{\cA}(\trainset, \unlearnset, \cT)}]$, which corresponds to our definition. 

    \begin{equation}\label{eq:strongly-convex-zeroth-order-final2}
    \begin{aligned}
        \norm{\bE\bs{\regret{T}{\cR_{\cA}(\trainset, \emptyset, \cT) }} - \bE\bs{\regret{T}{\cR_{\cA}(\trainset, \unlearnset, \cT) }}} &=  \norm{\sum_{i = 1}^k \sum_{t = \dt{i-1}}^{\dt{i}} f_t(z^\star) - f_t(z_i^\star)}\\
        &\leq \sum_{i = 1}^k \sum_{t = \dt{i-1}}^{\dt{i}} L\norm{z^\star - z_i^\star}\\
        &\overset{(a)}{\leq} \sum_{i = 1}^k \sum_{t = \dt{i-1}}^{\dt{i}} L\frac{2Li}{\mu T} = \frac{(1 + k) k L^2}{\mu},
    \end{aligned}
    \end{equation}
    where step (a) follows~\Cref{lem:stability-erm-multiple-points} and the fact that $\dt{i} - \dt{i-1}\leq T$. 

    \begin{lem}
        \label{lem:stability-erm-multiple-points}
        For a set of functions $f_1, ..., f_T$ and for any set of index $\{\di{i}\}_{i = 1}^k$ such that $\di{i}\leq T$, let $F(w) = \sum_{i = 1}^T f_i(w)$ and $F_i(w) = \sum_{i = 1}^T f_i(w) - \sum_{j = 1}^{i-1}f_{\di{i}}$. Let $w^\star = \argmin_w F(w)$ and $w^\star_i = \argmin_w F_i(w)$. If each $f_i$'s are $\mu$-strongly convex and $L$-Lipschitz, then \[\norm{w^\star - w_i^\star} \leq \frac{2(i-1)L}{\mu T}.\]
    \end{lem}

    Combining~\Cref{eq:strongly-convex-zeroth-order-final1,eq:strongly-convex-zeroth-order-final2}, and substituting $\omega = 1.2$ concludes the proof. 
\end{proof}

\begin{proof}[Proof of~\Cref{lem:stability-erm-multiple-points}]
    For any $i\in \{1, \ldots, k\}$ the definition of $F$, 
    \begin{equation}\label{eq:stability-lemma-multiple-points1}
    \begin{aligned}
            F(w_i^\star) &= F_i(w_i^\star) + \sum_{j = 1}^{i-1}f_{\di{i}}(w_i^\star)\\
            &\leq F_i(w^\star) + \sum_{j = 1}^{i-1}f_{\di{i}}(w_i^\star)\\
            &= F(w^\star) - \sum_{j = 1}^{i-1}f_{\di{i}}(w^\star) + \sum_{j = 1}^{i-1}f_{\di{i}}(w_i^\star)\\
            &\leq F(w^\star) + (i-1)L\norm{w^\star - w_i^\star},
    \end{aligned}
    \end{equation}
    where the last inequality follows by the Lipschitzness of all component functions $f_{\di{i}}$. 

    As each component functions $f_i's$ are $\mu$-strongly convex, the sum of $T$ such functions, $F$, is $T\mu$-strongly convex. Then, 
    \begin{equation}\label{eq:stability-lemma-multiple-points2}
        F(w_i^\star) \geq F(w^\star) + \frac{T\mu}{2} \norm{w_i^\star - w^\star}^2. 
    \end{equation}

    Combining~\Cref{eq:stability-lemma-multiple-points1,eq:stability-lemma-multiple-points2} concludes the proof. 
\end{proof}




\begin{corollary}
    Under the same assumption as~\Cref{thm:regret-passive-unlearner-strongly-convex},~\Cref{alg:zeroth-order-general} with step size $\eta_t = 1/(\mu t)$ satisfies 
    \[\bE\bs{\regret{T}{\trainset, \unlearnset, \cT}} \leq \frac{L^2}{\mu}\br{1 + k + k^2 + \log T} + \frac{d\omega k^{\omega + 1} L^2}{2\sqrt{3}e\ln\br{\frac{1}{\gamma}}\varepsilon \mu \br{\omega -1}}. \]
   \end{corollary}
   \begin{proof}
       We can show that \begin{equation}
       \begin{aligned}
           \frac{\dt{i} }{ \di{i}^2} \gamma^{2(\dt{i} - \di{i})} &=  \frac{\di{i} + (\dt{i} - \di{i})}{ \di{i}^2} \gamma^{2 (\dt{i} - \di{i})}\\
           &\leq \br{1 + (\dt{i} - \di{i})}\gamma^{2(\dt{i} - \di{i})}\leq 1 + \frac{1}{e\ln \frac{1}{\gamma}}
       \end{aligned}
       \end{equation}
   
       Substituting into the regret bound of~\Cref{thm:regret-passive-unlearner-strongly-convex}, \[\bE\bs{\regret{T}{\trainset, \unlearnset, \cT}} \leq \frac{L^2}{\mu}\br{1 + k + k^2 + \log T} + \frac{d\omega k^{\omega + 1} L^2}{2\sqrt{3}e\ln\br{\frac{1}{\gamma}}\varepsilon \mu \br{\omega -1}}. \]
   \end{proof}

\PassiveRegretCN*
    
    \begin{proof}
        We consider the same notation of $\cT$, $\cU$, $z_i^\star$ and $z^\star$ as in the proof of~\Cref{thm:regret-passive-unlearner-strongly-convex}. 
        
        Similar to~\Cref{eq:strongly-convex-zeroth-order1,eq:strongly-convex-zeroth-order2,eq:strongly-convex-zeroth-order3}, we can derive the following equation from the updating rule of OGD,  
        \begin{equation}\label{eq:convex-zeroth-order-non-adaptive3}
        (\nabla f_{t}(z_{t}))^\top \br{z_{t} - z^\star} \leq \begin{cases} 
          \frac{\norm{z_{t}-z^\star}^2 - \norm{z_{t+1}-z^\star}^2}{2\eta_{t}} + \frac{\eta_{t}\norm{\nabla  f_{t}(z_{t})}^2}{2} & t+1 \notin \cT \\
          - \nabla f_{t}(z_{t})^\top \xi_i + \frac{\norm{z_{t}-z^\star + \xi_i}^2 - \norm{z_{t + 1}-z^\star}^2}{2\eta_{t}} + \frac{\eta_{t}\norm{\nabla  f_{t}(z_{t})}^2}{2}  &  t +1 = \dt{i} \in \cT 
        \end{cases}
        \end{equation}

        As the loss functions are convex, for any $i\in \{1, \ldots, k\}$, \[ f_t(z_t) -  f_t(z^\star) \leq \br{\nabla  f_t (z_t)}^\top \br{z_t - z^\star}.\]

        Summing over $t\in \{1, \ldots, k\}$, we get the following regret bound with respect to a constant competitor, 
        \begin{equation*}
            \begin{aligned}
                 \regret{T}{\cR_{\cA}(S, \emptyset, \cT) } &= \sum_{t = 1}^T f_t(z_t) -  f_t(z^\star)\leq \sum_{t = 1}^T \br{\nabla  f_t(z_t)}^\top (z_t - z^\star) \\
                &\overset{(a)}{=} \underbrace{\sum_{t = 1}^T\frac{\norm{z_{t}-z^\star}^2 - \norm{z_{t+1}-z^\star}^2}{2\eta_{t}} + \frac{\eta_{t}\norm{\nabla  f_{t}(z_{t})}^2}{2} }_{A} \\
                & + \underbrace{\sum_{i = 1}^k - (\nabla f_{\dt{i}-1}(z_{\dt{i}-1}))^\top \xi_i + \frac{\br{z_{\dt{i} - 1}-z^\star}^\top \xi_i + \norm{\xi_i}^2}{2\eta_{\dt{i}-1}}}_{B}
            \end{aligned}
        \end{equation*}
        where step (a) follows by substituting~\Cref{eq:convex-zeroth-order-non-adaptive3}. 
        
        This way, we decompose the expected regret into two parts. In the following, we derive separate upper bounds for each.
        \begin{equation}\label{eq:convex-zeroth-order-non-adaptive4}
            \begin{aligned}
                \bE\bs{\regret{T}{\cR_{\cA}(S, \emptyset, \cT) }}&= \bE_{\xi_{1:k}}\bs{A + B}\\
                &= \bE_{\xi_{1:k}}\bE\bs{A|\xi_{1:k}} + \bE_{\xi_{1:k}}\bs{B}
            \end{aligned}
        \end{equation}
    
        We first bound the term $\bE\bs{A|\xi_{1:k}}$. Given $\xi_{1:k}$, $z_1, ..., z_T$ are deterministic. Thus, 
        \begin{equation*}
            \begin{aligned}
                \bE\bs{A|\xi_{1:k}} &= \sum_{t = 1}^T\frac{\norm{z_{t}-z^\star}^2 - \norm{z_{t+1}-z^\star}^2}{2\eta_{t}} + \frac{\eta_{t}\norm{\nabla  f_{t}(z_{t})}^2}{2} \\
                &= \bE\bs{\sum_{t = 1}^T \br{\frac{1}{\eta_t}-\frac{1}{\eta_{t-1}}}\norm{z_t - z^\star}^2 +  \sum_{t = 1}^T \frac{\eta_t \norm{\nabla  f_t(z_t)}^2}{2}}\\
                &\overset{(a)}{\leq} D^2\sum_{t = 1}^T\br{\frac{1}{\eta_t}-\frac{1}{\eta_{t-1}}} + \frac{L^2}{2}\sum_{t = 1}^T \eta_t\\
                &\overset{(b)}{\leq} DL \sqrt{T}+ \sum_{t = 1}^T \frac{L^2\eta_t}{2} \leq 3DL \sqrt{T}\\
            \end{aligned}
        \end{equation*}
        where step (a) is due to $\norm{z_t - z^\star}\leq D$ as $D$ is the diameter of the parameter space $\cK$ and Lipschitzness of the cost functions, and step (b) follows by substituting in specified learning rate $\eta_t = \frac{D}{L\sqrt{t}}$, and the last inequality follows by $\sum_{t = 1}^T \frac{1}{\sqrt{t}} \leq 2\sqrt{T}$.  
    
        Thus, 
        \begin{equation}\label{eq:strongly-convex-zeroth-order5}
            \bE_{\xi_{1:k}}\bs{A} = \bE_{\xi_{1:k}} \bE\bs{A|\xi_{1:k}} \leq 3 LD \sqrt{T}. 
        \end{equation}

        It remains to bound $\bE_{\xi_{1:k}}\bs{B}$, 
        
        \begin{equation}\label{eq:convex-zeroth-order-non-adaptive6}
            \begin{aligned}
                \bE_{\xi_{1:k}}\bs{B} &= \bE\bs{\sum_{i = 1}^k - (\nabla f_{\dt{i}-1}(z_{\dt{i}-1}))^\top \xi_i + \frac{2\br{z_{\dt{i} - 1}-z_i^\star}^\top \xi_i + \norm{\xi_i}^2}{2\eta_{\dt{i}-1}}}\\
                &\overset{(a)}{=}\bE\bs{\sum_{i = 1}^k \frac{ \norm{\xi_i}^2}{2\eta_{\dt{i}-1}}} = \sum_{i = 1}^k \frac{\bE\bs{\norm{\xi_i}}^2}{2\eta_{\dt{i}-1}} \\
                &\overset{(b)}{=} \sum_{i = 1}^k \frac{d\sigma_i^2}{2\eta_{\dt{i} - 1}} 
                \overset{(c)}{\leq} \frac{D L d\omega }{4(\omega-1)\varepsilon}\sum_{i = 1}^k\frac{i^{\omega}\sqrt{\dt{i}}}{\di{i}}
                \overset{(d)}{\leq} \frac{DL d\omega k^{\omega + 0.5}}{4\sqrt{3}(\omega-1)\varepsilon}\sqrt{\sum_{i = 1}^k\frac{\dt{i}}{\di{i}^2}}
            \end{aligned}   
        \end{equation}

        where step (a) is due to $\xi_i$'s independent from each other and have zero mean, step (b) follows by $\bE\bs{\norm{\xi_i}^2} = d\sigma_i^2$ for $\xi_i \sim \cN(0, \sigma_i^2 I_d)$. Step (c) follows by substituting in $\sigma_i^2 = \frac{i^{\omega}\omega L^2\eta_{\di{i}}^2}{2(\omega-1)\varepsilon}$, and step (d) follows an application of Cauchy-Schwarz inequality (similar to~\Cref{eq:bounds-on-omega-passive-sc}). 
        
        Combining~\Cref{eq:strongly-convex-zeroth-order4,eq:strongly-convex-zeroth-order5,eq:strongly-convex-zeroth-order6}, we have 
        \begin{equation}\label{eq:convex-zeroth-order-final-non-adaptive}
            \bE\bs{\regret{T}{\cR_{\cA}(\trainset,\emptyset, \cT) }} \leq 3DL \sqrt{T} + \frac{DLd\omega k^{\omega + 0.5}}{4\sqrt{3}(\omega-1)\varepsilon}\sqrt{\sum_{i = 1}^k\frac{\dt{i}}{\di{i}^2}} 
        \end{equation}
       
    Below, we bound the difference introduced by using a constant competitor: the distance between $\bE[\regret{T}{\cR_{\cA}(\trainset, \emptyset, \cT)}]$, which corresponds to the regret with a constant comparator, and $\bE[\regret{T}{\cR_{\cA}(\trainset, \unlearnset, \cT)}]$, which corresponds to our regret definition~\Cref{eq:regret-defn}. 
    \begin{equation}\label{eq:convex-zeroth-order-qg}
    \begin{aligned}
        \norm{\bE\bs{\regret{T}{\cR_{\cA}(\trainset, \emptyset, \cT) }} - \bE\bs{\regret{T}{\cR_{\cA}(\trainset, \unlearnset, \cT) }}} &=  \norm{\sum_{i = 1}^k \sum_{t = \dt{i}+1}^{\dt{i+1}} f_t(z^\star) - f_t(z_i^\star)}\\
        &\overset{(a)}{\leq} \sum_{i = 0}^k \sum_{t = \dt{i}+1}^{\dt{i+1}} L\norm{z^\star - z_i^\star}\\
        &\overset{(b)}{\leq} \sum_{i = 0}^k \sum_{t = \dt{i}+1}^{\dt{i+1}} L\frac{2Li}{\kappa \br{\dt{i+1}-\dt{i}}} \\
        &\leq \frac{L^2 \br{k^2 + k}}{\kappa}
    \end{aligned}
    \end{equation}
    where the second last inequality follows by the Quadratic Growth assumption of the composite loss function, and the stability of ERM under QG condition~\Cref{lem:stability-convex-qg}. 

    \begin{lem}[Stability for ERM under QG condition]
    \label{lem:stability-convex-qg}
        Let $F = \sum_{j = 1}^Tf_j$ and for any $k$ functions $\{i[1], \ldots, i[k]\}\subset \{1, \ldots T\}$, $\hat{F} = F - \sum_{\ell = 1}^k f_{i[\ell]}$. Let $z^\star$, $\hat{z}^\star$ be the ERM solution of $F, \hat{F}$ respectively. Assume the function $F$ satisfies quadratic growth with parameter $\kappa$, i.e. $F(z) - F(z^\star)\geq \frac{\kappa}{2}\norm{z - z^\star}^2$ for any $z$ and each component function $f_i$ are $L$-Lipschitz, then, \[\norm{z^\star - \hat{z}^\star} \leq \frac{2kL}{\kappa}\] 
    \end{lem}

    Finally, we choose $\omega = 1.2$, which concludes the proof. 
    \end{proof}

\begin{proof}[Proof of~\Cref{lem:stability-convex-qg}]
    \begin{equation}
        \begin{aligned}
            F(\hat{z}^\star) &= \hat{F}(\hat{z}^\star) + \sum_{\ell = 1}^k f_{i[\ell]}(\hat{z}^\star) \\
            &\leq \hat{F}(z^\star) + \sum_{\ell = 1}^k f_{i[\ell]}(\hat{z}^\star) \\
            &= F(z^\star) - \sum_{\ell = 1}^k f_{i[\ell]}(z^\star) + \sum_{\ell = 1}^k f_{i[\ell]}(\hat{z}^\star) \\
            &\leq F(z^\star) + kL\norm{z^\star - \hat{z}^\star}
        \end{aligned}
    \end{equation}
    where the last inequality follows the Lipschitzness of each $f_i$. 
    By the Quadratic Growth assumption,
    \begin{equation}
        F(\hat{z}^\star)-F(z^\star) \geq \frac{\kappa}{2}\norm{z^\star - \hat{z}^\star}^2
    \end{equation}
    Combining the two inequality complete the proof. 
\end{proof}

    
\PassiveRegretCA*

\begin{proof}
    For $i\in \{0, 1, ..., k\}$, let \[z_i^\star = \argmin_{z} \sum_{t = 1}^T  f_t(z) - \sum_{j = 1}^{i} f_{\di{i}} (z)\]
    be the best-in-hindsight estimator after ith deletion. Following a similar argument as in~\Cref{eq:strongly-convex-zeroth-order3}, 
    \begin{equation}
        \label{eq:sa-basic-ineq}
        (\nabla f_{t}(z_{t}))^\top \br{z_{t} - z^\star} \leq \begin{cases} 
        \frac{\norm{z_{t}-z^\star}^2 - \norm{z_{t+1}-z^\star}^2}{2\eta_{t}} + \frac{\eta_{t}\norm{\nabla  f_{t}(z_{t})}^2}{2} & t+1 \notin \cT \\
        - \nabla f_{t}(z_{t})^\top \xi_i + \frac{\norm{z_{t}-z^\star + \xi_i}^2 - \norm{z_{t + 1}-z^\star}^2}{2\eta_{t}} + \frac{\eta_{t}\norm{\nabla  f_{t}(z_{t})}^2}{2}  &  t +1 = \dt{i} \in \cT 
        \end{cases}
    \end{equation}
    
    By convexity of the loss function, for each $i$, and its corresponding time steps $t\in [\dt{i-1}, \dt{i}]$, 
    \begin{equation*}
         f_t(z_t) -  f_t(z_i^\star) \leq \nabla  f_t(z_t)^\top (z_t - z_i^\star) 
    \end{equation*}

    Summing over $t \in \{1, \ldots, T\}$ and substitute in~\Cref{eq:sa-basic-ineq}, we have 
    \begin{equation*}
        \begin{aligned}
            \regret{T}{\OnUnalg(\trainset,\unlearnset, \cT)} &= \sum_{i = 0}^k\sum_{t = \dt{i}}^{\dt{i + 1}}  f_t(z_t) -  f_t(z_i^\star) \\
            &\leq \underbrace{\sum_{i = 0}^k\sum_{t = \dt{i}}^{\dt{i+1}}  \frac{\norm{z_{t}-z_i^\star}^2 - \norm{z_{t+1}-z_i^\star}^2}{2\eta_t} + \frac{\eta_t\norm{\nabla  f(z_{t})}^2}{2}}_{A} \\
            & + \underbrace{\sum_{i = 1}^k - (\nabla f_{\dt{i}-1}(z_{\dt{i}-1}))^\top \xi_i + \frac{\br{z_{\dt{i} - 1}-z_i^\star}^\top \xi_i + \norm{\xi_i}^2}{2\eta_t}}_{B} 
        \end{aligned}
    \end{equation*}
    
    Then, the expected regret can be upper bounded as 
    \begin{equation}\label{eq:convex-zeroth-order1}
        \begin{aligned}
            \bE\bs{\regret{T}{\OnUnalg(\trainset,\unlearnset, \cT)}}&= \bE_{\xi_{1:k}}\bs{A + B}\\
            &= \bE_{\xi_{1:k}}\bE\bs{A|\xi_{1:k}} + \bE_{\xi_{1:k}}\bs{B}
        \end{aligned}
    \end{equation}

    We first bound the first term $\bE\bs{A|\xi_{1:k}}$. Given $\xi_{1:k}$, $z_1, ..., z_T$ are deterministic. Thus, 
    \begin{equation*}
        \begin{aligned}
           \bE\bs{A} &=  \bE\bs{A|\xi_{1:k}} = \bE\bs{\sum_{i = 0}^k\sum_{t = \dt{i}}^{\dt{i+1}}   \frac{\norm{z_{t}-z_i^\star}^2 - \norm{z_{t+1}-z_i^\star}^2}{2\eta_t} + \frac{\eta_t\norm{\nabla  f_t(z_{t})}^2}{2}}\\
            &\leq \sum_{i = 0}^k\sum_{t = \dt{i}}^{\dt{i+1}} \br{\frac{1}{\eta_t}-\frac{1}{\eta_{t-1}}}\frac{\norm{z_t - z_i^\star}^2}{2}+ \sum_{t = 1}^T \frac{\eta_t\norm{\nabla  f_t(z_t)}^2}{2} \\
            &\leq  \frac{D^2}{2\eta_T}  + \sum_{t = 1}^T \frac{D\norm{\nabla  f_t(z_t)}^2}{2\sqrt{\sum_{j = 1}^t \norm{\nabla  f_j(z_j)}^2}}\\
            &\overset{(a)}{\leq} \frac{D}{2}\sqrt{\sum_{t = 1}^T \norm{\nabla  f_t(z_t)}^2} + D\sqrt{\sum_{t = 1}^T \norm{\nabla  f_t(z_t)}^2} = \frac{3D}{2}\sqrt{\sum_{t = 1}^T \norm{\nabla  f_t(z_t)}^2}
        \end{aligned}
    \end{equation*}
    where step (a) follows by~\Cref{lem:adaptive-learning-rate}.
    
    \begin{lem}[\citet{orabona2019modern}]
        \label{lem:adaptive-learning-rate}
        Let $a_0 \geq 0$ and $f:[0, \infty]\to[0, \infty]$ a nonincreasing function. Then, \[\sum_{t = 1}^T a_tf\br{a_0 + \sum_{i = 1}^t a_i}\leq \int_{a_0}^{\sum_{t = 0}^T a_t}f(x)dx.\]
    \end{lem}
    

    As the loss functions are $\beta$-smooth, we apply~\Cref{lem:properties-of-smooth-functions} to arrive an upper bound on part A. 
    \begin{lem}[Lemma 4.1 in~\citet{srebro2012optimisticrateslearningsmooth}]\label{lem:properties-of-smooth-functions}
        If $f$ is a $\beta$-smooth function, then the following holds, 
        \[\norm{\nabla f(x)}_*^2\leq 2\beta \bs{f(x) - \inf_{y \in \bR^d} f(y)}. \]
    \end{lem}

    \begin{equation}\label{eq:convex-zeroth-order2}
    \begin{aligned}
        \bE\bs{A} &= \bE\bs{A|\xi_{1:k}} \leq \frac{3D}{2}\sqrt{\sum_{t = 1}^T \norm{\nabla  f_t (z_t)}^2}\\
        &\overset{(a)}{\leq} \frac{3D}{2}\sqrt{\beta\sum_{t = 1}^T \bs{ f_t(z_t) - \inf_{y \in \bR^d } f_t(y)}}\overset{(b)}{\leq} \frac{3D}{2}\sqrt{\beta\sum_{t = 1}^T  f_t(z_t)}
    \end{aligned}
    \end{equation}
    where step (a) follows~\Cref{lem:properties-of-smooth-functions}, and step (b) follows by non-negativity of the loss functions. 



    It remains to bound $\bE_{\xi_{1:k}}\bs{B}$, 

    \begin{equation}\label{eq:convex-zeroth-order3}
        \begin{aligned}
            \bE_{\xi_{1:k}}\bs{B} &= \bE\bs{\sum_{i = 1}^k - (\nabla f_{\dt{i}-1}(z_{\dt{i}-1}))^\top \xi_i + \frac{\br{z_{\dt{i} - 1}-z^\star}^\top \xi_i + \norm{\xi_i}^2}{2\eta_{\dt{i}-1}}}\\
            &\overset{(a)}{=}\bE\bs{\sum_{i = 1}^k \frac{ \norm{\xi_i}^2}{2\eta_{\dt{i}-1}}} \overset{(b)}{=} \sum_{i = 1}^k \frac{d i^\omega \omega L^2 \eta_{\di{i}}^2}{4(\omega-1)\varepsilon \eta_{\dt{i}}}\\
            &\overset{(c)}{\leq} \frac{d k^{\omega+0.5}\omega L^2}{4\sqrt{3}(\omega - 1) \varepsilon}\sqrt{\sum_{i = 1}^k\frac{\eta_{\di{i}}^4}{\eta_{\dt{i}}^2}}\overset{(d)}{=}\frac{d k^{\omega+0.5}\omega L^2D }{4\sqrt{3}(\omega - 1) \varepsilon}\sqrt{\sum_{i = 1}^k\frac{\sum_{j = 1}^{\dt{i}}\norm{\nabla f_j(z_j)}^2}{\br{\sum_{j = 1}^{\di{i}}\norm{\nabla f_j(z_j)}^2}^2}} 
        \end{aligned}   
    \end{equation}
    where step (a) follows from the fact that $\xi_i$ are independent and have zero mean. Step (b) follows by $\bE\bs{\norm{\xi}^2} = d\sigma_i^2$ for $\xi_i\sim \cN(0, \sigma_i^2)$ where $\sigma_i^2 = \frac{i^\omega \omega L^2\eta_{\di{i}}^2}{2(\omega - 1)\varepsilon}$. Step (c) follows by Cauchy-Schwarz inequality and step (d) follows by substituting in $\eta_t = \frac{D}{\sqrt{\sum_{i = 1}^t\norm{f_i(z_i)}^2}}$. 
    
    Combining~\Cref{eq:convex-zeroth-order1,eq:convex-zeroth-order2,eq:convex-zeroth-order3}, we have 
    \begin{equation}
        \sum_{i = 0}^k\sum_{t = \dt{i}+1}^{\dt{i+1}}  f_t(z_t) -  f_t(z_i^\star)  \leq \bE\bs{B}
            + \frac{3D}{2}\sqrt{\beta\sum_{t = 1}^T f_t (z_t)}
    \end{equation}


    Then, we apply~\Cref{lem:implicit-bounds} with $x = \sum_{t = 1}^Tf_t(z_t)$ and $a = \frac{9D^2\beta}{4}$,  $b = 0$ and $c = \sum_{i = 0}^k \sum_{t = \dt{i}+1}^{\dt{i+1}} f_t (z_i^\star) + \bE\bs{B}$
    \begin{lem}
        \label{lem:implicit-bounds}
        Let $a, c > 0$, $b\geq 0$, and $x\geq 0$ such that $x - \sqrt{ax + b}\leq c$. Then $x\leq a + c + 2\sqrt{b + ac}$. 
    \end{lem}

    Therefore, the regret is upper bounded by 
    \begin{equation}
        \begin{aligned}
            \sum_{i = 0}^k\sum_{t = \dt{i}+1}^{\dt{i + 1}} \br{ f_t(z_t) -  f_t(z_i^\star)} &\leq \frac{9D^2\beta}{4} + \bE\bs{B} + 3D\sqrt{\beta \sum_{i = 0}^k \sum_{t = \dt{i}+1}^{\dt{i+1}}f_t(z_i^\star) + \bE\bs{B}}\\
            &\leq \frac{9D^2\beta}{4} + 3D\sqrt{\sum_{i = 0}^k\sum_{t = \dt{i}+1}^{\dt{i+1}}   f_t(z_i^\star)} \\
            &+ \frac{d k^{\omega+0.5}\omega L^2D^2 }{4(\omega - 1) \varepsilon}\sqrt{\sum_{i = 1}^k\frac{\beta \sum_{j = 1}^{\dt{i}}\norm{\nabla f_j(z_j)}^2}{\br{\sum_{j = 1}^{\di{i}}\norm{\nabla f_j(z_j)}^2}^2}} .
        \end{aligned}
    \end{equation}
    Taking $\omega = 1.5$ concludes the proof. 
\end{proof}



\begin{proof}[Proof of~\Cref{lem:implicit-bounds}]
Starting with $x - c \leq \sqrt{ax + b}$, we can square both sides and get 
\[x^2 - \br{2c + a} x + c^2 \leq b.\]
Completing the square, \[\br{x - \frac{2c + a}{2}}^2 \leq b + ac + \frac{a^2}{4}\] 
Taking square root and rearrange terms, we get 
\begin{equation}
    \begin{aligned}
    x &\leq \frac{2c + a}{2} + \sqrt{ b + ac + \frac{a^2}{4}}\\
    &\leq \frac{2c + a}{2} + \sqrt{b + ac} + \sqrt{\frac{a^2}{4}}\\
    &= a + c + \sqrt{b + ac}
    \end{aligned}
\end{equation}This concludes the proof. 
\end{proof}

\PassiveRegretWorstCase*
\begin{proof}
    Similar to the previous proof, from the update rule of OGD, we have 
    \begin{equation}\label{eq:worst-case-update}
        (\nabla f_{t}(z_{t}))^\top \br{z_{t} - z_i^\star} =\begin{cases} 
          \frac{\norm{z_{t}-z_i^\star}^2 - \norm{z_{t+1}-z_i^\star}^2}{2\eta} + \frac{\eta\norm{\nabla  f_{t}(z_{t})}^2}{2} & t+1 \notin \cT \\
          - \nabla f_{t}(z_{t})^\top \xi_i + \frac{\norm{z_{t}-z_i^\star + \xi_i}^2 - \norm{z_{t + 1}-z_i^\star}^2}{2\eta} + \frac{\eta\norm{\nabla  f_{t}(z_{t})}^2}{2}  &  t +1 = \dt{i} \in \cT 
        \end{cases}
    \end{equation}
    Therefore, using the property of convex functions, we can upper bound the expected regret by \begin{equation}
        \begin{aligned}
            \bE\bs{\sum_{t}^T f_t(z_t) - f_t(z^\star)} &\leq \bE\bs{\sum_{t}^T \br{\nabla f_t(z_t)}^\top \br{z_t - z^\star}}\\
            &\overset{(a)}{\leq} \frac{1}{2}\bs{\sum_{t = 1}^T \frac{\norm{z_t - z^\star}^2 - \norm{z_{t+1}-z^\star}^2}{\eta} + \eta \norm{\nabla f_t(z_t)}^2 + \sum_{i = 1}^k \eta \bE\norm{\xi_i}^2}\\
            &\overset{(c)}{\leq} \frac{1}{2}\bs{\sum_{t = 1}^T \frac{2D^2}{\eta} + \eta \br{L^2 + \sum_{i = 1}^k \frac{L^2 d j^\omega \omega}{2(\omega -1)\varepsilon}}}\\
            &\overset{(d)}{\leq}\frac{1}{2}\bs{\frac{2D^2}{\eta} + \eta T\br{L^2 + \frac{\omega k^{\omega+1}L^2d}{2(\omega^2-1) \varepsilon}}}
        \end{aligned}
    \end{equation}
    where step (a) follows~\Cref{eq:worst-case-update}, step (c) follows by Lipschitzness of the cost functions and the boundedness of the parameter space, and step (d) follows by the fact that \[\sum_{j = 1}^k j^\omega \leq \int_{0}^k x^\omega dx \leq \frac{k^{\omega + 1}}{\omega + 1}.\] 
    Setting $\eta = \sqrt{\frac{2D^2}{T\br{L^2 + \frac{\omega k^{\omega + 1}L^2d}{2(\omega^2-1)\varepsilon}}}}$, we arrive at the upper bound 
    \[\bE\bs{\sum_{t}^T f_t(z_t) - f_t(z^\star)}\leq \sqrt{T\br{2D^2L^2 + \frac{\omega k^{\omega + 1}L^2d}{2(\omega^2-1)\varepsilon}}}\leq \br{DL + \frac{L}{2}\sqrt{\frac{\omega k^{\omega + 1}d}{(\omega^2-1)\varepsilon}}}\sqrt{2T}. \]
\end{proof}
\section{Omitted Proofs for~\Cref{sec:active-unlearner}}\label{app:ERM}


\begin{algorithm}[htbp]
    \caption{First-order active online learner and unlearner}
    \label{alg:ERM-first-order}
    \begin{algorithmic}[1]
        \REQUIRE Cost functions $f_1, ..., f_T$ that are $L$-Lipschitz, learning rates $\eta_1, ..., \eta_T$, a deletion time set $\cT$, a deletion index set $\cU$, and privacy parameter $\varepsilon$, the auxiliary unlearner $\Unaux$ and its Lipschitz parameter $L_R$.
        \STATE Initialize $z_1\in \cK$. 
        \FOR{Time step $t = 2, ... T$}
            \STATE Set $z_t = \Pi_{\cK}\bs{z_{t-1} - \eta_t \nabla f_{t-1}(z_{t-1})}$
            \IF{there exists $\dt{i}\in \cT$ such that $t = \dt{i}$}
            \STATE Set $\hat{z}_{\dt{i}}= \Pi_{\cK}\bs{z_{\dt{i}-1} - \eta_{t} \nabla f_{\dt{i}-1}(z_{t-1})}$
            \STATE Starting from $\hat{z}_{\dt{i}}$, $\cI_{1, i}$ steps of gradient descent with learning rate $\frac{1}{\beta + \mu}$ on all points till now $f_{1:\dt{i}}$ and output $z_{\dt{i}}'$
            \STATE Starting from $z_{\dt{i}}'$, perform $\cI_{2}$ steps of gradient descent with learning rate $\frac{1}{\beta + \mu}$ on all remaining data points $f_{i:\dt{i}}\setminus \{f_{\di{1}}:f_{u_{i}}\}$ and output $z_{\dt{i}}''$. 
            \STATE Set $z_{\dt{i}} = z_{\dt{i}}'' + \xi_i$, where $\xi_i \sim \cN(0, \sigma_i^2)$ for 
            \begin{equation}
                \label{eq:sigma-i-active-first-order}\sigma_i = \gamma^{\cI_2}\sqrt{\frac{ i^\omega \omega}{2(\omega -1)\varepsilon}}\frac{L\br{6i + L\gamma^{\dt{i} - \di{i}}\eta_{\di{i}}}}{\dt{i} \mu}.
            \end{equation}
        \ENDIF
        \STATE Output $z_t$ 
        \ENDFOR
    \end{algorithmic}
\end{algorithm}

\FirstOrderGuarantee*
\begin{proof}[Unlearning guarantee in~\Cref{thm:first-order-guarantee}]
    For any $i$, we express~\Cref{alg:ERM-first-order}'s output at time $\dt{i}$ as two CNIs~(\Cref{defn:CNI}). Let $\trainsetsub{\dt{i}} = \{f_1, \ldots, f_{\dt{i}}\}$, and define $\trainsetsub{\dt{i}}' = \{f_1', \ldots, f_{\dt{i}}'\}$ as the same set with functions at $\di{j}$, $j\leq i$ replaced with $\perp$. 

    Then, $\psi_j$ denote the standard online gradient descent steps between $\dt{j-1}$ and $\dt{j}$ using points in $\trainsetsub{\dt{i}}$. Formally, define the updating functions $g_t(z) = \Pi_{\cK}\bs{z - \eta_t \nabla f_t(z_t)}$ and $g_t'(z) = \Pi_{\cK}\bs{z - \eta_t \nabla f_t'(z_t)}$ with respect to cost functions $\cF$ and $\cF'$ respectively, then $\psi_j$ and $\psi_j'$ are defined as follows, 
    \[\psi_j = g_{t_{j}}\circ\ldots \circ g_{\dt{j-1}}, \quad \psi_j' = g_{t_{j}}'\circ\ldots \circ g_{\dt{j-1}}' .\]
    We denote the deterministic unlearning procedure (Step 6 and 7 in~\Cref{alg:ERM-first-order}) as the function $\Unaux$ that takes as input the current output $\hat{z}_{\dt{i}}$, past cost functions $\trainsetsub{\dt{i}} = \{f_1,..., f_{\dt{i}}\}$ and the deleted functions $\unlearnsetsub{i} = \{f_{\di{1}}, \ldots, f_{\di{i}}\}$. Then, the CNIs are 
    \[\bc{z_0, \bc{\Unaux(\psi_j, \trainsetsub{\dt{i}}, \unlearnsetsub{i})}_{j = 1}^i, \bc{\zeta_j}_{j = 1}^i}, \quad \bc{z_0, \bc{\Unaux(\psi_j', \trainsetsub{\dt{i}}\setminus \unlearnsetsub{i},\emptyset )}_{j = 1}^i, \bc{\zeta_j}_{j = 1}^i}\]
    Next, we apply~\Cref{lem:PAI} to upper bound $\rdp{z_{\dt{i}}}{z_{\dt{i}}'}$. To do this, we need to derive a sequence of $a_1, ..., a_{\dt{i}}$ such that $e_t = \sum_{j = 1}^t \gamma^{t - j}\br{s_j - a_j} \geq 0$ for all $t\in \{1, ..., \dt{i}\}$, and $e_{\dt{i}} = 0$, where $s_j$ is defined as follows,
    \begin{equation}
        \begin{aligned}
            s_j &= \max_{z\in \cK}\norm{\Unaux(\psi_j(z), \trainsetsub{\dt{j}}, \unlearnsetsub{j}) - \Unaux(\psi_j'(z), \trainsetsub{\dt{j}}\setminus \unlearnsetsub{j},\emptyset )}_2
        \end{aligned}
    \end{equation}

   As we assume $\di{i} \in (\dt{i-1}, \dt{i})$ for all $i$, 
   \begin{equation}\small
    \begin{aligned}
        s_j &= \max_{z\in \cK}\norm{\Unaux(\psi_j(z), \trainsetsub{\dt{j}}, \unlearnsetsub{j}) - \Unaux(\psi_j'(z), \trainsetsub{\dt{j}}\setminus \unlearnsetsub{j},\emptyset )}_2\\
        &= \max_{z\in \cK}\norm{\Unaux(\psi_j(z), \trainsetsub{\dt{j}}, \unlearnsetsub{j}) - \Unaux(\psi_j(z), \trainsetsub{\dt{j}}\setminus \unlearnsetsub{j},\emptyset ) + \Unaux(\psi_j(z), \trainsetsub{\dt{j}}\setminus \unlearnsetsub{j},\emptyset ) - \Unaux(\psi_j'(z), \trainsetsub{\dt{j}}\setminus \unlearnsetsub{j},\emptyset )}_2\\
        &= \max_{z\in \cK}\norm{\Unaux(z, \trainsetsub{\dt{j}}, \unlearnsetsub{j}) - \Unaux(z, \trainsetsub{\dt{j}}\setminus \unlearnsetsub{j},\emptyset )}_2 + L_R(\psi_j(z) - \psi_j(z)')\\
        & =\max_{z\in \cK}\norm{\Unaux(\psi_j(z), \trainsetsub{\dt{j}}, \unlearnsetsub{j}) - \Unaux(\psi_j(z), \trainsetsub{\dt{j}}\setminus \unlearnsetsub{j},\emptyset )}_2 + L_R \gamma^{\dt{i} - \di{i}}\eta_{\di{i}}L
    \end{aligned}
   \end{equation}
    where $L_R$ is the Lipschitzness coefficient of the unlearning algorithm $\Unaux$. 
   
   Setting $a_i = s_i$ and $\sigma_j^2 = \frac{\omega j^\omega a_j^2}{2(\omega -1)\varepsilon}$, we have the following guarantee: for all $i\in [k]$, \[\rdp{z_{\dt{i}}}{z_{\dt{i}}'}\overset{(a)}{\leq} \sum_{j = 1}^i R(\zeta_{\dt{j}}, a_j) = \sum_{j = 1}^i \frac{\alpha \varepsilon (\omega-1)}{\omega j^\omega} \leq \alpha\varepsilon. \]
    where (a) follows~\Cref{lem:PAI}. 
    
   It remains to compute $L_R$ and $\max_{z\in \cK}\norm{\Unaux(\psi_j(z), \trainsetsub{\dt{j}}, \unlearnsetsub{j}) - \Unaux(\psi_j(z), \trainsetsub{\dt{j}}\setminus \unlearnsetsub{j},\emptyset )}_2$ for $\Unaux$. 

    Due to the contractiveness of gradient descent under strongly convex and smooth cost functions~(\Cref{lem:gd-convergence-sc-smooth}), our unlearning algorithm \(\Unaux\), which consists of \(\cI_1 + \cI_2\) contractive steps with contraction coefficient \(\gamma\), is Lipschitz continuous with constant \(L_R = \gamma^{\cI_1 + \cI_2}\).

    \begin{lemL}[Convergence of gradient descent \cite{chen20lecture_notes}]
        \label{lem:gd-convergence-sc-smooth}
        If the loss function $\ell$ is $\mu$-strongly convex and $\beta$-smooth, then the output $w_t$ of $T$-step gradient descent on $S$ with learning rate $\eta = \frac{2}{\mu + \beta}$ and initialization $w_0$ satisfies \[\norm{w_t - w^\star}_2\leq \gamma^T \norm{w_0 -w^\star}_2,  \]
        where $\gamma = \frac{\beta/\mu - 1}{\beta/\mu + 1}$ and $w^\star = \textrm{arg}\min_w \sum_{x\in S}\ell(w, x)$ is the minimizer of the loss function $\ell$ on the dataset $S$. 
    \end{lemL}

    Next, we derive an upper bound on $\max_{z\in \cK}\norm{\Unaux(\psi_j(z), \trainsetsub{\dt{j}}, \unlearnsetsub{j}) - \Unaux(\psi_j(z), \trainsetsub{\dt{j}}\setminus \unlearnsetsub{j},\emptyset )}_2$. Consider the first \(\cI_1\) gradient descent (GD) steps of the unlearning algorithm on the set \(\trainsetsub{\dt{j}}\), denoted by the function \(F_0\), followed by \(\cI_2\) GD steps on the set \(\trainsetsub{\dt{j}} \setminus \unlearnsetsub{j}\), denoted by \(F_1\). The unlearning auxiliary function \(\Unaux(\psi_j(z), \trainsetsub{\dt{j}}, \unlearnsetsub{j})\) performs these two phases: \(\cI_1\) GD steps on \(\trainsetsub{\dt{j}}\) (via \(F_0\)) and \(\cI_2\) GD steps on \(\trainsetsub{\dt{j}} \setminus \unlearnsetsub{j}\) (via \(F_1\)).
    
    Similarly, \(\Unaux(\psi_j(z), \trainsetsub{\dt{j}} \setminus \unlearnsetsub{j}, \emptyset)\) performs \(\cI_1 + \cI_2\) GD steps entirely on the set \(\trainsetsub{\dt{j}} \setminus \unlearnsetsub{j}\), denoted by \(F_0\) and \(\tilde{F}_1\), respectively.


    Then, \begin{equation}
        \begin{aligned}
            &\max_{z\in \cK}\norm{\Unaux(\psi_j(z), \trainsetsub{\dt{j}}, \unlearnsetsub{j}) - \Unaux(\psi_j(z), \trainsetsub{\dt{j}}\setminus \unlearnsetsub{j},\emptyset )}_2 \\
            =& \max_z \norm{F_0\circ F_1(z) - F_0\circ  \tilde{F}_1(z) }_2\\
            \overset{(a)}{\leq}& \gamma^{\cI_2} \norm{F_1 (z) - \tilde{F}_1(z)}\\
            =& \gamma^{\cI_2}\norm{\br{F_1(z) - ERM_1 + ERM_1 - ERM_0 + ERM_0 - \tilde{F}_1(z)}}\\
            \overset{(c)}{\leq} & \frac{2(j+1)\gamma^{\cI_2}L}{\dt{i} \mu}
        \end{aligned}
    \end{equation}
    where step (a) follows by the contractiveness of the GD, or the convergence bound of GD~\Cref{lem:gd-convergence-sc-smooth}, $ERM_1$ and $ERM_0$ represents the ERM of all points $ \trainsetsub{\dt{j}}$ and $ \trainsetsub{\dt{j}}\setminus \unlearnsetsub{j}$ respectively. Step (c) follows by the contractiveness of GD~(\Cref{lem:gd-convergence-sc-smooth}), stability of ERM for strongly convex functions (\Cref{lem:stability-erm-multiple-points}) and the lower bound on $\cI_1$ by $\log (L/\mu D \dt{j})/\log \gamma$. 

\end{proof}
\begin{proof}[Regret guarantee in~\Cref{thm:first-order-guarantee}]
    Consider the sequence of $z_t$ output by active online learner and unlearner $\cA$, i.e. 

\begin{equation}
    \begin{aligned}
        z_{t+1} &= \Pi_{\cK}\bs{z_{t} - \eta_{t} \nabla f_{t} (z_{t})}  & t+1\notin\cT\\
        z_{t+1} &= \Unaux(\Pi_{\cK}\bs{z_{t} - \eta_{t} \nabla f_{t} (z_{t})}, f_{1:t}, f_{\di{1}:\di{i}}) + \xi_i  & t+1 = \dt{i}\in\cT
    \end{aligned}
\end{equation}
Similar as in the previous proof, let \[z_{i, 0}^\star = \argmin_{z\in \cK}\sum_{t = 1}^{\dt{i}} f_t(z) - \sum_{j = 1}^{i} f_{\di{i}}(z)\]
\begin{equation}
    \begin{aligned}
        \norm{z_{t+1}-z^\star}^2 &\leq \norm{z_{t}-z^\star}^2 + \eta_t^2\norm{\nabla f_t(z_t)}^2 + 2\eta_t (\nabla f_t(z_t))^\top \br{z_t - z^\star}, &t + 1\notin \cT\\ 
        \norm{z_{t+1}-z^\star}^2 &\leq \norm{z_{i, 0}^\star - z^\star + d_i + \xi_i}^2, & t+1 = \dt{i} \in \cT  
    \end{aligned}
\end{equation}
where \[\norm{\Pi_{\cK}\bs{\Unaux(z_{\dt{i}-1} - \eta_{\dt{i}-1} \nabla f_{\dt{i}-1} (z_{\dt{i}-1})}, f_{1:\dt{i}}, f_{\di{1}:\di{i}}) - z_{0, i}^\star} \leq \frac{2\gamma^{\cI_2} DL}{\dt{i}\mu } =:d_i\]for step 5 and 6 in~\Cref{alg:ERM-first-order} following~\Cref{lem:gd-convergence-sc-smooth}. 

Let $t_0 = 0, t_{k+1} = T$. By strong convexity of the cost function, 
\begin{equation}
    \begin{aligned}
       \regret{T}{\cA, \cT, \cU} &= \sum_{t = 1}^T f_t(z_t) - f_t(z^\star) = \sum_{i = 0}^{k} \sum_{t = t_{i}+1}^{\dt{i+1}} f_t(z_t) - f_t(z^\star)\\
        &\leq \underbrace{\sum_{i = 0}^{k} \sum_{t = t_{i}+1}^{\dt{i+1}-1} \br{\nabla f_t(z_t)}^\top \br{z_t - z^\star} - \frac{\mu}{2}\norm{z_t - z^\star}^2}_{A} + \underbrace{\sum_{i = 1}^{k} L\norm{z_{\dt{i}} - z^\star}}_{B},
    \end{aligned}
\end{equation}

We first bound part B
\begin{equation}\label{eq:first-order-regret-B}
    \begin{aligned}
        B &= \sum_{i= 1}^{k}L \norm{z_{i, 0}^\star + \xi_i + d_i - z^\star}\\
        &\leq L \sum_{i = 1}^k \norm{\xi_i} + d_i + \norm{z_{i, 0}^\star - z^\star} \\
    \end{aligned}
\end{equation}

By Jensen's inequality and stability of ERM~(\Cref{lem:stability-erm-multiple-points}), 
\[\bE \bs{B} \leq  \sum_{i = 1}^k \sigma_i + d_i + \frac{L(T-i-\dt{i})}{T\mu} = O\br{\sum_{i = 1}^k\frac{Ld}{\dt{i}\mu\varepsilon}+\frac{kL}{\mu}},\]
where $\sigma_i$ is the standard deviation of $\xi_i$ defined in~\Cref{eq:sigma-i-active}, and the last equality follows by $\cI_2\geq \br{2.2\log k}/{\log\frac{1}{\gamma}}$. 

It remains to bound 2 times part A, we first note that when $t = \dt{i} - 1$, the unlearning takes place, 
\begin{equation}\label{eq:first-order-regret-A}
    \begin{aligned}
        2A &=\sum_{i = 0}^{k} \sum_{t = t_{i}+1}^{\dt{i+1}-1} \br{\nabla f_t(z_t)}^\top \br{z_t - z^\star} - \frac{\mu}{2}\norm{z_t - z^\star}^2 \\
        &=\underbrace{\sum_{i = 0}^{k} \sum_{t = t_{i}+1}^{\dt{i+1}-2} \frac{\norm{z_t -z^\star}^2 - \norm{z_{t+1}-z^\star}^2}{\eta_t} -\mu \norm{z_t - z^\star}^2}_{C} + \underbrace{\sum_{i = 0}^{k} \sum_{t = t_{i}+1}^{\dt{i+1}-2} \eta_t \norm{\nabla f_t (z_t)}^2}_{D} \\
        &+ \underbrace{\sum_{i = 1}^k (\nabla f_{t_{i}-1}(z_{\dt{i}-1}))^\top \br{z_{\dt{i}-1} - z^\star}}_{E}
    \end{aligned}
\end{equation}

Part D is upper bounded by $O(\log T)$, 
\begin{equation}
    \label{eq:first-order-regret-D}
    D \leq \sum_{t = 1}^T \eta_t L^2 =\sum_{t = 1}^T  \frac{L^2}{t\mu} = O(\log T).
\end{equation}

Part E is upper bounded by \begin{equation}\label{eq:first-order-regret-E}
    E\leq 2kLD 
\end{equation}
Next, we bound part C, 
\begin{equation}\label{eq:first-order-regret-C1}
    \begin{aligned}
        C &= \sum_{i = 0}^{k} \norm{z_{\dt{i} + 1} - z^\star}^2\br{\frac{1}{\eta_{\dt{i}+1}} -\mu} - \frac{1}{\eta_{\dt{i+1}-1}}\norm{z_{\dt{i+1}-1}-z^\star}^2 \\
        &\leq \sum_{i = 0}^{k} \norm{z_{\dt{i} + 1} - z^\star}^2\dt{i} \mu - t _{i+1}\mu \norm{z_{\dt{i+1}-1}-z^\star}^2 + \mu \sum_{i = 1}^k \norm{z_{\dt{i} - 1} - z^\star}^2\\
        &\overset{(a)}{\leq}  \sum_{i = 1}^k \dt{i}\mu\br{\norm{z_{\dt{i}+1}-z^\star}^2 - \norm{z_{\dt{i}-1} - z^\star}^2}  + \mu kD^2
    \end{aligned}
\end{equation}

It remains to upper bound the term $\norm{z_{\dt{i}+1}-z^\star}^2 - \norm{z_{\dt{i}-1} - z^\star}^2$. We use the following inequality~(\Cref{lem:ineq-squared-diff}), 
\begin{lem}\label{lem:ineq-squared-diff}
    For any $a, b, c\in \bR^d$, 
    \begin{equation}
    \norm{a - b} - \norm{c-b}^2 \leq \norm{a-c}^2 + 2\norm{a-c}\norm{c-b}
\end{equation}
\end{lem}

As the unlearning algorithm moves $z_{\dt{i}}$ close to ERM solution without the deleted points $z_{i, 0}^\star$, and that $z_{\dt{i}+1}$ is obtained by doing one gradient descent step on $z_{\dt{i}}$, we have 
\begin{equation}
\begin{aligned}
        \norm{z_{\dt{i-1}} - z^\star_{i-1,0}} &\leq \gamma^{\cI_2}\br{\gamma^{\cI_1}D + \frac{L}{\mu \dt{i-1}}} = \frac{\gamma^{\cI_2}L}{\mu \dt{i} }\\
        \norm{z_{\dt{i}+1}-z_{i, 0}^\star} = \norm{z_{\dt{i}}-\eta_{\dt{i}}\nabla f_{\dt{i}}(z_{\dt{i} }) - z^\star} &\leq \norm{z_{\dt{i}} - z^\star_{i, 0}} + \frac{L}{\dt{i} \mu} \leq  \frac{2\gamma^{\cI_2}L}{\mu \dt{i} }
\end{aligned}
\end{equation}

Consider the composition of update functions of OGD with learning rate $\eta_t = \frac{1}{\mu t}$ on the functions $f_{\dt{i-1}}, ..., f_{\dt{i} + 1}$ as $\textrm{GD}$, then $\textrm{GD}$ is $\gamma^{\dt{i} + 1 - \dt{i-1}}$-contractive. By~\Cref{assump:assumption2}, $a_i$ is a fixed point of $\mathrm{GD}$ and is close to the ERM solution $z_{i, 0}^\star$, i.e.$GD(z_{i, 0}^\star) = z_{i, 0}^\star$ and $\norm{z_{i, 0}^\star - a_i}\leq \frac{1}{\dt{i}}$. Then, we can upper bound $\norm{z_{\dt{i} + 1} - z_{\dt{i}-1}}$ by 
\begin{equation}
    \begin{aligned}
        \norm{z_{\dt{i} + 1} - z_{\dt{i} -1}} &\leq \norm{z_{i, 0}^\star + \frac{\gamma^{\cI_2}L}{\mu \dt{i}} - GD\br{z_{i-1, 0}^\star + \frac{\gamma^\cI_2L}{\mu \dt{i-1}}}} \\
        &\leq \norm{a_i + \frac{\mu + \gamma^{I_2L}}{\mu \dt{i}}- GD\br{z_{i-1, 0}^\star + \frac{\gamma^\cI_2L}{\mu \dt{i-1}}}}\\
        &= \norm{GD(a_i) + \frac{\gamma^{\cI_2}L}{\mu \dt{i}} - GD\br{z_{i-1, 0}^\star + \frac{\gamma^{\cI_2}L}{\mu \dt{i-1}}}} \\
        &\leq \gamma^{\dt{i} - \dt{i-1}} \norm{a - z_{i-1, 0}^\star + \frac{\gamma^{\cI_2}L}{\mu \dt{i-1}}} + \frac{
        \mu + \gamma^{\cI_2}L}{\mu t_{i}} \\
        &\leq \gamma^{\dt{i} - \dt{i-1}}\norm{z_{i, 0}^\star - z_{i-1, 0}^\star + \frac{\mu + \gamma^{\cI_2}L}{\mu \dt{i-1}}} + \frac{\mu + \gamma^{\cI_2}L}{\mu t_{i}} \\
        &= \frac{\gamma^{\dt{i} - \dt{i-1}}(\dt{i} - \dt{i-1})L}{\mu \dt{i}} + \frac{2\br{\mu + \gamma^{\cI_2}L}}{\mu \dt{i}}
    \end{aligned}
\end{equation}
Substituting this inequality into~\Cref{eq:first-order-regret-C1}, we get an upper bound\begin{equation}\label{eq:first-order-regret-C}
    C \leq \sum_{i = 1}^k \gamma^{\dt{i}-\dt{i-1}}\br{\dt{i} - \dt{i-1}}L + k\mu (2\gamma^{\cI_2}L + D^2)\leq \frac{kL}{e\ln\frac{1}{\gamma}} + k\mu (2\gamma^{\cI_2}L + D^2)
\end{equation}
as $\gamma^{\dt{i}-\dt{i-1}}\br{\dt{i} - \dt{i-1}}\leq \frac{1}{e\ln\frac{1}{\gamma}}$. 

Combining the upper bounds on part A, B, C, D, E,(\Cref{eq:first-order-regret-A,eq:first-order-regret-B,eq:first-order-regret-C1,eq:first-order-regret-D,eq:first-order-regret-E}) we can upper bound the regret by
\begin{equation}
   O\br{\log T + kLD^2 + \frac{Lkd}{\mu \varepsilon} + \frac{kL}{e\ln \frac{1}{\gamma}}}
\end{equation}

The above regret bound is measured against a fixed competitor, as in previous proofs. Next, we adjust the analysis to account for a dynamic competitor, reflecting changes in the best-in-hindsight estimator after each deletion.
\begin{equation}
    \begin{aligned}
        \norm{\bE\bs{\regret{T}{\cR_{\cA}(S, \emptyset, \cT) }} - \bE\bs{\regret{T}{\cR_{\cA}(S,  S_U, \cT) }}} &=  \norm{\sum_{i = 1}^k \sum_{t = \dt{i-1}}^{\dt{i}} f_t(z^\star) - f_t(z_i^\star)}\\
        &\leq \sum_{i = 1}^k \sum_{t = \dt{i-1}}^{\dt{i}} L\norm{z^\star - z_i^\star}\\
        &\overset{(a)}{\leq} \sum_{i = 1}^k \sum_{t = \dt{i-1}}^{\dt{i}} L\frac{2Li}{\mu T} = \frac{(1 + k) k L^2}{\mu},
    \end{aligned}
\end{equation}
This complete the proof. 
\end{proof}

\begin{proof}[Proof of~\Cref{lem:ineq-squared-diff}]
\begin{equation}
    \begin{aligned}
        \norm{a-b}^2 - \norm{c-b}^2
        &= \norm{a}^2 + \norm{b}^2 - 2a^\top b - \norm{c}^2 - \norm{b}^2 + 2 c^\top b\\
        &= \norm{a}^2 - 2a^\top b - \norm{c}^2+ 2 c^\top b\\
        &= (a-c)^\top (a+c-2b) = (a-c)^\top (a-c+2(c-b))\\
        &= \norm{a-c}^2 + 2(a-c)^\top (c-b) \leq \norm{a-c}^2 + 2\norm{a-c} \norm{c-b}\\
    \end{aligned}
\end{equation}    
\end{proof}

\begin{algorithm}
    \caption{Second-order active online learner and unlearner}
    \label{alg:ERM-second-order}
    \begin{algorithmic}[1]
        \REQUIRE Cost functions $f_1, ..., f_T$ that are $L$-Lipschitz, learning rates $\eta_1, ..., \eta_T$, a deletion time set $\cT$, a deletion index set $\cU$, and privacy parameter $\varepsilon$, the auxiliary unlearner $\Unaux$ and its Lipschitz parameter $L_R$ and error functions $e_2, e_1$. 
        \STATE Initialize $z_1\in \cK$. 
        \FOR{Time step $t = 2, ... T$}
            \STATE Set $z_t = z_{t-1} - \eta_t \nabla f_{t-1}(z_{t-1})$
            \IF{there exists $\dt{i}\in \cT$ such that $t = \dt{i}$}
            \STATE Starting from $\hat{z}_{\dt{i}}$, $\cI_1$ steps of gradient descent with learning rate $\frac{1}{\beta + \mu}$ on all points till now $f_{1:\dt{i}}$ and output $z_{\dt{i}}'$
            \STATE Set $z_{\dt{i}}' = \hat{z_{\dt{i}}} - \frac{1}{\dt{i} - i}\hat{H}^{-1} \sum_{j = 1}^i \nabla f_{u_j}(\hat{z}_{\dt{i}})$, where \[\hat{H}^{-1} = \frac{1}{\dt{i} - i}\br{\sum_{j = 1}^{\dt{i}} \nabla^2 f_j(\hat{z}_{\dt{i}}) - \sum_{j = 1}^i}\nabla^2 f_{u_j}(\hat{z}_{\dt{i}})\]
            \STATE Set $z_{\dt{i}} = z_{\dt{i}}' +  \xi_i$, where $\xi_i \sim \cN(0, \sigma_i^2)$ for 
            \begin{equation}
                \label{eq:sigma-i-active}\sigma_i = \sqrt{\frac{\alpha j^\omega \omega}{2(\omega -1)\varepsilon}}\frac{L\br{2 + \frac{k\beta}{\mu (\dt{i} - k)}\br{\frac{M}{\mu}-1}}}{\mu \br{\dt{i} - i}}.
            \end{equation}
        \ENDIF
        \STATE Output $z_t$ 
        \ENDFOR
    \end{algorithmic}
\end{algorithm}

\end{document}